\documentclass[11p]{article} 
\def\arXiv{1}

\newcommand{\notarxiv}[1]{foo}
\newcommand{\arxiv}[1]{ba}
\ifdefined \arXiv
	\renewcommand{\arxiv}[1]{#1}%
	\renewcommand{\notarxiv}[1]{\ignorespaces}%
\else%
	\renewcommand{\arxiv}[1]{\ignorespaces}%
	\renewcommand{\notarxiv}[1]{#1}%
\fi%

\arxiv{
	\usepackage[top=1in, right=1in, left=1in, bottom=1in]{geometry}
	\usepackage{url}
	\usepackage[hidelinks]{hyperref}
	\hypersetup{
		colorlinks=true,
		linkcolor=blue!70!black,
		citecolor=blue!70!black,
		urlcolor=blue!70!black}
	\usepackage[numbers, square]{natbib}
}

\notarxiv{
	\usepackage[subtle, all=normal, mathspacing=normal, floats=tight, 
	sections=tight]{savetrees}
	\usepackage[hidelinks=true]{hyperref}
}

\usepackage{amsmath,amsfonts,bm}

\def\secref#1{section~\ref{#1}}

\def\eqref#1{equation~\ref{#1}}

\def\1{\bm{1}}

\def\rvu{{\mathbf{i}}}

\def\rvm{{\mathbf{m}}}
\def\rvn{{\mathbf{n}}}

\def\rvu{{\mathbf{u}}}
\def\rvv{{\mathbf{v}}}

\def\rvx{{\mathbf{x}}}

\def\rmG{{\mathbf{G}}}

\def\rmI{{\mathbf{I}}}

\def\rmM{{\mathbf{M}}}

\def\rmP{{\mathbf{P}}}

\def\rmR{{\mathbf{R}}}

\def\rmU{{\mathbf{U}}}
\def\rmV{{\mathbf{V}}}

\def\rmZ{{\mathbf{Z}}}

\def\vmu{{\bm{\mu}}}

\def\vv{{\bm{v}}}
\def\vw{{\bm{w}}}
\def\vx{{\bm{x}}}

\def\mP{{\bm{P}}}

\def\mBeta{{\bm{\beta}}}

\DeclareMathAlphabet{\mathsfit}{\encodingdefault}{\sfdefault}{m}{sl}
\SetMathAlphabet{\mathsfit}{bold}{\encodingdefault}{\sfdefault}{bx}{n}

\def\sP{{\mathbb{P}}}

\newcommand{\E}{\mathbb{E}}

\newcommand{\R}{\mathbb{R}}

\usepackage{comment}
\usepackage{amsmath}
\usepackage{amssymb}
\usepackage{amsthm}
\usepackage{mathtools}
\usepackage{enumitem}
\usepackage{xspace}

\newtheorem{definition}{Definition}
\newtheorem*{definition*}{Definition}
\newtheorem*{lemma*}{Lemma}
\newtheorem*{theorem*}{Theorem}
\newtheorem{lemma}{Lemma}
\newtheorem{corollary}{Corollary}
\newtheorem{theorem}{Theorem}
\newtheorem{proposition}{Proposition}
\newtheorem{assumption}{Assumption}

\newtheorem{claim}{Claim}

\newcommand{\idxset}{\bar{S}}

\newcommand{\normspu}{r_s}
\newcommand{\normcore}{r_c}

\newcommand{\bv}{{\mathbf v}}
\newcommand{\x}{{\mathbf x}}

\newcommand{\z}{{\mathbf z}}
\newcommand{\w}{{\mathbf w}}
\newcommand{\ones}{{\mathbf 1}}

\newcommand{\G}{{\mathbf{G}}}
\newcommand{\I}{{\mathbf{I}}}

\newcommand{\succesprobpos}{$1-34\exp(-t^2)$ }

\newcommand{\reals}{\mathbb{R}}
\newcommand{\Y}{\mathcal{Y}}

\newcommand{\cX}{\mathcal{X}}

\newcommand{\F}{\mathcal{F}}
\newcommand{\N}{\mathcal{N}}
\newcommand{\cL}{\mathcal{L}}

\newcommand{\half}{\frac{1}{2}}

\newcommand{\defeq}{\coloneqq}

\newcommand{\learningsetting}{Linear Two Environment Problem\xspace}

\DeclareMathOperator*{\maximize}{maximize}

\DeclarePairedDelimiterX{\inp}[2]{\langle}{\rangle}{#1, #2}
\DeclarePairedDelimiterX{\abs}[1]{|}{|}{#1}
\DeclarePairedDelimiterX{\norm}[1]{\|}{\|}{#1}

\renewcommand{\eqref}[1]{Equation~(\ref{#1})}
\newcommand{\thmref}[1]{Theorem ~\ref{#1}}
\newcommand{\lemref}[1]{Lemma~\ref{#1}}

\newcommand{\corref}[1]{Corollary~\ref{#1}}
\newcommand{\assumref}[1]{Assumption~\ref{#1}}

\newlength\myindent 
\makeatletter
\newcommand*{\indep}{%
  \mathbin{%
    \mathpalette{\@indep}{}%
  }%
}
\newcommand*{\nindep}{%
  \mathbin{
    \mathpalette{\@indep}{\not}
  }%
}
\newcommand*{\@indep}[2]{%
  \sbox0{$#1\perp\m@th$}
  \sbox2{$#1=$}
  \sbox4{$#1\vcenter{}$}
  \rlap{\copy0}
  \dimen@=\dimexpr\ht2-\ht4-.2pt\relax
  \kern\dimen@
  {#2}%
  \kern\dimen@
  \copy0 
} 
\makeatother

\renewcommand{\P}{\mathbb{P}}

\usepackage{setspace}
\usepackage{xcolor}

\newcommand{\yw}[1]{{\color{blue}}}
\newcommand{\gy}[1]{{\color{purple}}}
\newcommand{\us}[1]{{\color{cyan}}}
\newcommand{\yc}[1]{{\color{green!50!black}}}
\newcommand{\ywsuggestion}[1]{{\color{brown}}}

\usepackage[hidelinks]{hyperref}
\usepackage{url}
\usepackage{algorithm}
\usepackage{algpseudocode}
\usepackage{cleveref}
\usepackage{wrapfig}
\usepackage{amsthm}

\makeatletter
\newtheorem*{rep@theorem}{\rep@title}
\newcommand{\newreptheorem}[2]{%
\newenvironment{rep#1}[1]{%
 \def\rep@title{#2 \ref{##1}}%
 \begin{rep@theorem}}%
 {\end{rep@theorem}}}
\makeatother

\newreptheorem{theorem}{Theorem}
\newreptheorem{proposition}{Proposition}

\algrenewcommand\algorithmicrequire{\textbf{Input:}}
\algrenewcommand\algorithmicensure{\textbf{Output:}}

\title{Malign Overfitting: Interpolation Can Provably Preclude Invariance}

\author{Yoav Wald\thanks{
Johns Hopkins University,
\texttt{ywald1@jhu.edu}}
~~
Gal Yona \thanks{
Weizmann Institute of Science}
~~
Uri Shalit \thanks{
Technion}
~~
Yair Carmon \thanks{
Tel Aviv University}
}

\date{}

\begin{document}

\maketitle

\begin{abstract}
Learned classifiers should often possess certain invariance properties meant to encourage fairness, robustness, or out-of-distribution generalization. 
However, multiple recent works empirically demonstrate that common invariance-inducing regularizers are ineffective in the over-parameterized regime, in which classifiers perfectly fit (i.e. interpolate) the training data. This suggests that the phenomenon of ``benign overfitting," in which models generalize well despite interpolating, might not favorably extend to settings in which robustness or fairness are desirable. 

In this work, we provide a theoretical justification for these observations. We prove that---even in the simplest of settings---any interpolating learning rule (with an arbitrarily small margin) will not satisfy these invariance properties. We then propose and analyze an algorithm that---in the same setting---successfully learns a non-interpolating classifier that is provably invariant. We validate our theoretical observations  on simulated data and the Waterbirds dataset.



\end{abstract}

\section{Introduction}
Modern machine learning applications often call for models which are not only accurate, but which are also robust to distribution shifts or satisfy fairness constraints. For example, we might wish to avoid using hospital-specific traces in X-ray images \citep{degrave2021ai, zech2018variable}, as they rely on spurious correlations that will not generalize to a new hospital, or we might seek ``Equal Opportunity'' models attaining similar error rates across protected demographic groups, e.g., in the context of loan applications \citep{byanjankar2015predicting,hardt2016equality}. A developing paradigm for fulfilling such requirements is learning models that satisfy some notion of \emph{invariance} \citep{peters2016causal, peters2017elements} across environments or sub-populations. For example, in the X-ray case, spurious correlations can be formalized as relationships between a feature and a label which vary across hospitals \citep{zech2018variable}. Equal Opportunity \citep{hardt2016equality} can be expressed as a statistical constraint on the outputs of the model, where the false negative rate is invariant to membership in a protected group. 
Many techniques for learning invariant models have been proposed
including
penalties that encourage invariance \citep{arjovsky2019invariant, krueger2020out, veitch2021counterfactual, wald2021calibration, puli2021out,makar2021causally,pmlr-v162-rame22a,kaur2022modeling}, data re-weighting \citep{sagawa2019distributionally, wang2021is, idrissi2022simple}, causal graph analysis \citep{subbaswamy2019preventing, subbaswamy2022unifying}, and more \citep{ahuja2020invariant}.



While the invariance paradigm holds promise for delivering robust and fair models, many current invariance-inducing methods often fail to improve over naive approaches. This is especially noticeable when these methods are used with overparameterized deep models capable of \emph{interpolating}, i.e., perfectly fitting the training data \citep{gulrajani2020search, NEURIPS2021_972cda1e, guo2022evaluation, zhou2022sparse, menon2021overparameterisation, veldanda2022fairness,cherepanova2021technical}. Existing theory explains why overparameterization hurts invariance for standard interpolating learning rules, such as empirical risk minimization and max-margin classification~\citep{pmlr-v119-sagawa20a,nagarajan2020understanding,d2020underspecification}, and also why reweighting and some types of distributionally robust optimization face challenges when used with overparameterized models~\citep{pmlr-v97-byrd19a,sagawa2019distributionally}.
In contrast, training overparameterized models to interpolate the training data typically results in good \emph{in-distribution} generalization, and such ``benign overfitting'' \citep{kini2021label, wang2021is} is considered a key characteristic of modern deep learning~\citep{cao2021risk, wang2021benign, pmlr-v178-shamir22a}. \yc{todo: better refs here?} 
Consequently, a number of works attempt to extend benign overfitting to robust or fair generalization by designing new interpolating learning rules~\citep{cao2019learning, kini2021label, wang2021is,lu2022importance}. 

In this paper, we demonstrate that such attempts face a fundamental obstacle, because \emph{all} interpolating learning rules (and not just maximum-margin classifiers) fail to produce invariant models in certain high-dimensional settings where invariant learning (without interpolation) is possible. This \emph{does not} occur because there are no invariant models that separate the data, but because interpolating learning rules \emph{cannot find them}. In other words, beyond identically-distributed test sets, overfitting is no longer benign. More concretely, we consider linear classification in a basic overparameterized Gaussian mixture model with invariant ``core'' features as well as environment-dependent ``spurious'' features, similar to models used in previous work to gain insight into robustness and invariance \citep{schmidt2018adversarially, rosenfeld2020risks, pmlr-v119-sagawa20a}. We show that any learning rule producing a classifier that separates the data with non-zero margin must necessarily rely on the spurious features in the data, and therefore cannot be invariant. Moreover, in the same setting we analyze a simple two-stage algorithm that can find accurate and nearly invariant linear classifiers, i.e., with almost no dependence on the spurious feature. 

Thus, we establish a separation between the level of invariance attained by interpolating and non-interpolating learning rules.
We believe that learning rules which fail in the simple over-parameterized linear classification setting we consider are not likely to succeed in more complicated, real-world settings. Therefore, our analysis provides useful guidance for future research into robust and fair machine learning models, as well as theoretical support for the recent success of non-interpolating robust learning schemes \citep{rosenfeld2022erm,veldanda2022fairness, kirichenko2022last,menon2021overparameterisation, pmlr-v180-kumar22a, pmlr-v162-zhang22u,idrissi2022simple,chatterji2022under}.

\textbf{Paper organization.} \yc{I think the paragraph needs to start with something like ``paper organization,'' since otherwise it's not immediately clear what it's about} The next section formally states our full result (\Cref{thm:main_statement}). In \Cref{sec:negative_result} we outline the arguments leading to the negative part of \Cref{thm:main_statement}, i.e., the failure of interpolating classifiers to be invariant in our model. In \Cref{sec:positive_result} we establish the positive part \Cref{thm:main_statement}, by providing and analyzing a non-interpolating algorithm that, in our model, achieves low robust error.
We validate our theoretical findings with simulations and experiments on the Waterbirds dataset in \Cref{sec:validation}, \yc{todo: mention additional experiments if they exist} and conclude with a discussion of additional related results and directions for future research in \Cref{sec:discussion}.

\section{Statement of Main Result} \label{sec:overview}

\subsection{Preliminaries}

\paragraph{Data model.}
Our analysis focuses on learning linear models over covariates $\x$ distributed as a mixture of two Gaussian distributions corresponding to the label $y$. 
\begin{definition} \label{def:environment}
An \emph{environment} is a distribution parameterized by $(\vmu_c, \vmu_s, d, \sigma, \theta)$ where $\theta\in{[-1, 1]}$ and $\vmu_c, \vmu_s\in{\reals^d}$ satisfy $\vmu_c \perp \vmu_s$ and with 
samples 
generated according to:
\begin{align} 
\label{eq:DGP}
\sP_{\theta}(y) = \textnormal{Unif}{\{-1, 1\}}, \quad 
\sP_{\theta}(\x|y) = \N(y\vmu_{c}+y\theta\vmu_{s}, \sigma^{2}I).
\end{align}
\end{definition}

Our goal is to find a (linear) classifier that predicts $y$ from $\x$ and is robust to the value of $\theta$ (we discuss the specific robustness metric below). To do so, the classifier will need to have significant inner product with the ``core'' signal component $\vmu_c$ and be approximately orthogonal to the ``spurious'' component $\vmu_s$. We focus on learning problems where we are given access to samples from two environments that share all their parameters other than $\theta$, as we define next. We illustrate our setting with \Cref{fig:projections} in \Cref{sec:app_helpers}.

\begin{definition}[\learningsetting] 
\label{def:learning_setup}
In a \learningsetting\ we have datasets $S_1=\{\x_i^{(1)},y_i^{(1)}\}_{i=1}^{N_1}$ and $S_2=\{\x_i^{(2)},y_i^{(2)}\}_{i=1}^{N_2}$ \yc{todo maybe: fanicer notation than $S_i$} of sizes $N_1, N_2$ drawn from $\sP_{\theta_1}$ and $\sP_{\theta_2}$ respectively. A learning algorithm is a (possibly randomized) mapping from the tuple $(S_1, S_2)$ to a linear classifier $\w\in{\reals^d}$. We let  $S=\{\x_i,y_i\}_{i=1}^{N}$ denote that dataset pooled from $S_1$ and $S_2$ where $N=N_1+N_2$. 
Finally we let $r_c\defeq\norm{\vmu_c}$ and $r_s\defeq\norm{\vmu_s}$.

\end{definition}
We study settings where $\theta_1, \theta_2$ are fixed and $d$ is large compared to $N$, i.e.\ the overparameterized regime.
We refer to the two distributions $\sP_{\theta_e}$ for $e\in{\{1,2\}}$ as ``training environments", following \cite{peters2016causal, arjovsky2019invariant}. In the context of Out-of-Distribution (OOD) generalization, environments correspond to different experimental conditions, e.g., collection of medical data in two hospitals. In a fairness context, we may think of these distributions as subpopulations (e.g., demographic groups).\footnote{We note that in some settings, more commonly in the fairness literature, $e$ is treated as a feature given to the classifier as input. Our focus is on cases where this is either impossible or undesired. For instance, because at test time $e$ is unobserved or ill-defined (e.g. we obtain data from a new hospital). However, we emphasize that the \emph{leaning rules} we consider have full knowledge of which environment produced each training example}
While these are  different applications that require specialized methods, the underlying formalism of solutions is often similar \cite[see, e.g.,][Table~1]{pmlr-v139-creager21a}, where we wish to learn a classifier that in one way or another is invariant to the environment variable. 

\paragraph{Robust performance metric.} 
An advantage of the simple model defined above is that many of the common invariance criteria all boil down to the same mathematical constraint: learning a classifier that is orthogonal to $\vmu_s$, which induces a spurious correlation between the environment and the label. These include  Equalized Odds \citep{hardt2016equality}, conditional distribution matching \cite{li2018deep}, calibration on multiple subsets of the data \citep{hebert2018multicalibration, wald2021calibration}, Risk Extrapolation \citep{krueger2020out} and CVaR fairness \citep{pmlr-v97-williamson19a}.
 
 In terms of predictive accuracy, the goal of learning a linear model that aligns with $\vmu_c$ (the invariant part of the data generating process for the label) and is orthogonal to $\vmu_s$
coincides with providing guarantees on the robust error, i.e. the error when data is generated with values of $\theta$ that are different from the $\theta_1, \theta_2$ used to generate training data.\footnote{In fact, as we show in \Cref{eq:error_explicit} in \Cref{sec:negative_result}, learning a model orthogonal to $\vmu_s$ is also a necessary condition to minimize the robust error. Thus, attaining guarantees on the robust error also has consequences on invariance of the model, as defined by these criteria. We discuss this further in \secref{sec:invariance_defs} of the appendix.}
\begin{definition}[Robust error]
The robust error of a linear classifier  $\w\in\R^d$ is:
\begin{align} \label{eq:robust_error}
    \max_{\theta\in{[-1, 1]}}{\epsilon_{\theta}(\w)}, \ \text{where} \ \epsilon_{\theta}(\w):={\E_{\substack{\x,y\sim \sP_{\theta}}}\left[ \mathrm{sign}(\inp{\w}{\x}) \neq y \right]}.
\end{align}
\end{definition}

\paragraph{Normalized margin.}
We study is whether algorithms that perfectly fit (i.e.\ interpolate) their training data can learn models with low robust error.  Ideally, we would like to give a result on all classifiers that attain training error zero in terms of the $0$-$1$ loss. However, the inherent discontinuity of this loss would make any such statement sensitive to instabilities and pathologies. For instance, if we do not limit the capacity of our models, we can turn any classifier into an interpolating one by adding ``special cases" for the training points, yet intuitively this is not the type of interpolation that we would like to study. To avoid such issues, we replace the 0-1 loss with a common continuous surrogate, the normalize margin, and require it to be strictly positive.

\begin{definition}[Normalized margin] \label{def:margin}
Let $\gamma>0$, we say a classifier $\w\in{\reals^d}$ separates the set $S=\{\x_i, y_i\}_{i=1}^{N}$ with \emph{normalized margin} $\gamma$ if for every $(\vx,y)\in S$
\begin{align*}
    \frac{y_i\inp{\w}{\x_i}}{\norm{\w}} > \gamma\sqrt{\sigma^2 d}.
\end{align*}
\end{definition}
The $\sqrt{\sigma^2 d}$ scaling of $\gamma$ is roughly proportional to $\norm{\x}$ under our data model in \eqref{eq:DGP}, and keeps the value of $\gamma$ comparable across growing values of $d$. 

\subsection{Main Result}

Equipped with the necessary definitions, we now state and discuss our main result.

\begin{theorem}\label{thm:main_statement}
For any sample sizes $N_1,N_2>65$, margin lower bound $\gamma \le \frac{1}{4\sqrt{N}}$, target robust error $\epsilon > 0$, and coefficients $\theta_1=1$, $\theta_2 > -\frac{N_1 \gamma}{\sqrt{288N_2}}$, there exist parameters $r_c,r_s>0$, $d>N$, and $\sigma>0$ such that the following holds for the \learningsetting (\Cref{def:learning_setup}) with these parameters.
\begin{enumerate}[leftmargin=*]
    \item \textbf{Invariance is attainable}. \Cref{alg:two_phase_learning} maps $(S_1, S_2)$ to a linear classifier ${\w}$ such that with probability at least $\frac{99}{100}$ (over the draw $S$), the \emph{robust error} of $\w$ is less than $\epsilon$.
    \item \textbf{Interpolation is attainable}. With probability at least $\frac{99}{100}$, the  estimator $\w_{\mathrm{mean}}=N^{-1}\sum_{i\in[N]}{y_i\x_i}$ separates $S$ with normalized margin (\Cref{def:margin}) greater than $\frac{1}{4\sqrt{N}}$.
    \vspace{-5pt}
    \item \textbf{Interpolation precludes invariance.} Given $\vmu_c$ uniformly distributed on the sphere of radius $r_c$ and $\vmu_s$ uniformly distributed on a sphere of radius $r_s$ in the subspace orthogonal to $\vmu_c$, let $\w$ be any classifier learned from $(S_1,S_2)$ as per \Cref{def:learning_setup}. If ${\w}$ separates $S$ with normalized margin $\gamma$, then with probability at least $\frac{99}{100}$ (over the draw of $\vmu_c, \vmu_s$, and the sample), the \emph{robust error} of $\w$ is at least $\frac{1}{2}$.
\end{enumerate}
\end{theorem}

\Cref{thm:main_statement} shows that if a learning algorithm for overparameterized linear classifiers always separates its training data, then there exist natural settings for which the algorithm completely fails to learn a robust classifier, and will therefore fail on multiple other invariance and fairness objectives.
Furthermore, in the same setting this failure is avoidable, as there exists an algorithm (that \emph{necessarily} does not always separate its training data) which successfully learns an invariant classifier. This result has deep implications for theoreticians attempting to prove finite-sample invariant learning guarantees: it shows that---in the fundamental setting of linear classification---no interpolating algorithm can have guarantees as strong as non-interpolating algorithms such as \Cref{alg:two_phase_learning}.

Importantly, \Cref{thm:main_statement} requires interpolating invariant classifiers to \emph{exist}---and shows that these classifiers are information-theoretically \emph{impossible to learn}. In particular, the first part of the theorem implies that the Bayes optimal invariant classifier $\vw=\vmu_c$ has robust test error at most $\epsilon$. Therefore, for all $\epsilon < \frac{1}{100N}$ we have that $\vmu_c$ interpolates $S$ with probability $> \frac{99}{100}$. Furthermore, a short calculation (see \Cref{ssec:invaraint-margin-calculation}) shows that (for $r_c$, $r_s$, $d$ and $\sigma$ satisfying \Cref{thm:main_statement}) the normalized margin of $\vmu_c$ \yc{probably not quite right} is
 $\Omega(({N+\sqrt{N_2}/\gamma})^{-\frac{1}{2}})$. 
However, we prove that---due to the high-dimensional nature of the problem---no algorithm can use $(S_1,S_2)$ to reliably distinguish the invariant interpolator from other interpolators with similar or larger margin. This learnability barrier strongly leverages our random choice of $\vmu_c$, $\vmu_s$, without which the (fixed) vector $\vmu_c$ would be a valid learning output.

We establish \Cref{thm:main_statement} with three propositions, each corresponding to an enumerated claim in the theorem: (1) \Cref{prop:positive_result_main} (in \S\ref{sec:positive_result}) establishes that invariance is attainable, (2) \Cref{prop:achievable_margin} (\Cref{sec:achievable_margin}) establishes that interpolation is attainable, and (3) \Cref{prop:neg_result_details} (in \S\ref{sec:negative_result}) establishes that interpolation precludes invariance. We choose to begin with the latter proposition since it is the main conceptual and technical contribution of our paper. Conversely, \Cref{prop:achievable_margin} is an easy byproduct of the developments leading up to \Cref{prop:neg_result_details}, and we defer it to the appendix. 

With \Cref{prop:neg_result_details,prop:positive_result_main,prop:achievable_margin} in hand, the proof of \Cref{thm:main_statement}  simply consists of choosing the free parameters in \Cref{thm:main_statement} ($r_c,r_s,d$ and $\sigma$) based on these propositions such that all the claims in the theorem hold simultaneously. For convenience we take $\sigma^2 = 1/d$. Then (ignoring constant factors) we pick $r_s^2 \propto \frac{1}{N}$ and $r_c^2 \propto r_s^2 / (1+ \frac{\sqrt{N_2}}{N_1 \gamma})$ in order to satisfy requirements in \Cref{prop:neg_result_details,prop:achievable_margin}. Finally, we take $d$ to be sufficiently large so as to satisfy the remaining requirements, resulting in $d\propto \max\left\{ N^2, \frac{N}{\gamma^2 N_1^2 r_c^2}, \frac{(Q^{-1}(\epsilon))^2}{N_{\min} r_c^4}, \frac{1}{N_{\min}^2 r_c^4}\right\}$, where $N_{\min} = \min\{N_1,N_2\}$ and $Q$ is the Gaussian tail function (see \Cref{sec:app_main_statement_proof} for the full proof).

We conclude this section with remarks on the range of parameters under which \Cref{thm:main_statement} holds. The impossibility results in \Cref{thm:main_statement} are strongest when $N_2$ is smaller than $N_1^2 \gamma^2$. In particular, when $N_2 \le N_1^2 \gamma^2 / 288$, our result holds for all $\theta_2\in[-1,1]$ and moreover the core and spurious signal strengths $r_c$ and $r_s$ can be chosen to be of the same order.
The ratio $N_2/(N_1^2\gamma^2)$ is small \emph{either} when one group is under-represented (i.e., $N_2 \ll N_1$) \emph{or} when considering large margin classifiers (i.e., $\gamma$ of the order $1/\sqrt{N}$). Moreover, unlike prior work on barriers to robustness \citep[e.g.,][]{pmlr-v119-sagawa20a,nagarajan2020understanding}, our result continue to hold even for balanced data and arbitrarily low margin, provided $\theta_2$ is close to $0$ and the core signal is sufficiently weaker than the spurious signal. Notably, the normalized margin $\gamma$ can be arbitrarily small while the maximum achievable margin is always at least of the order of $\frac{1}{\sqrt{N}}$. Therefore, we believe that \Cref{thm:main_statement} essentially precludes any interpolating learning rule from being consistently invariant.


\yc{todo: add here or in the next section a discussion of why we have to make $\vmu_c,\vmu_s$ random. this is a question I often get for proofs of this sort so it would be good to explain}

\section{Interpolating Models Cannot Be Invariant} \label{sec:negative_result}


In this section we prove the third claim in \Cref{thm:main_statement}:
for essentially any nonzero value of the normalized margin $\gamma$, there are instances of the \learningsetting\ (\Cref{def:learning_setup}) where with high probability, learning algorithms that return linear classifiers attaining normalized margin at least $\gamma$ must incur a large robust error. 
The following proposition formalizes the claim; we sketch the proof below and provide a full derivation in \Cref{sec:prop1_full_proof}.

\begin{proposition} \label{prop:neg_result_details}

For $\sigma=1/\sqrt{d}$, $\theta_1=1$, there are universal constants $c_r\in(0,1)$ and $C_d, C_r\in(1,\infty)$, such that, for any target normalized $\gamma$, $\theta_2 > - N_1\gamma / \sqrt{288 N_2}$, and failure probability $\delta\in(0,1)$, if 
\begin{align} 
    &\max\{r_s^2, r_c^2\} \le \frac{c_r}{N}
    ~~,~~
    \frac{r_s^2}{r_c^2} \ge C_r \left( 1 + \frac{\sqrt{N_2}}{N_1\gamma}\right)
    ~~\mbox{ and}~~\label{eq:negative_res_norm_constraint}\\
    &d \geq C_d \frac{N}{\gamma^2 N_1^2 r_c^2}\log \frac{1}{\delta},
    \label{eq:negative_res_dim_constraint}
\end{align}
then with probability at least $1-\delta$ over the drawing of $\vmu_c, \vmu_s$ and $(S_1, S_2)$ as described in \thmref{thm:main_statement}, any $\hat{\w}\in{\reals^d}$ that is a measurable function of $(S_1, S_2)$ and separates the data \yc{we should be clearer about what ``any model'' and ``the data'' means. Theorem 1 does a reasonable job I think but we should either repeat this here or refer to the appropriate formal definition} with normalized margin larger than $\gamma$ has robust error at least $0.5$.
\end{proposition}

\begin{proof}[Proof sketch]
\yc{I rewrote this part - please take a look}
We begin by noting that for any fixed $\theta$, the error of a linear classifier $\w$ is 
\begin{align} \label{eq:error_explicit}
    \epsilon_{\theta}(\w) = 
    Q\left(\frac{\inp{\w}{\vmu_c} + \theta\inp{\w}{\vmu_s}}{\sigma \norm{\w}}\right) = Q\left(\frac{\inp{\w}{\vmu_c}}{\sigma \norm{\w}}\left( 1 + \theta\frac{\inp{\w}{\vmu_s}}{\inp{\w}{\vmu_c}} \right)\right),
\end{align}
where $Q(t) \coloneqq \P(\mathcal{N}(0;1) > t)$ is the Gaussian tail function. 
Consequently, when $\inp{\w}{\vmu_s} / \inp{\w}{\vmu_c} \ge 1$ it is easy to see that $\epsilon_{\theta}(\w)=1/2$ for some $\theta\in[-1,1]$ and therefore the robust error is at least $\half$; we prove that $\inp{\w}{\vmu_s} / \inp{\w}{\vmu_c} \ge 1$ indeed holds with high probability under the proposition's assumptions. Our proof has two key parts: (a) restricting the set of classifiers to the linear span of the data and (b) lower bounding the minimum value of $\inp{\w}{\vmu_s} / \inp{\w}{\vmu_c}$ for classifier in that linear span. 

For the first part of the proof we use the spherical distribution of $\vmu_c$ and $\vmu_s$ and concentration of measure to show that (with high probability) any component of $\w$ chosen outside the linear span of $\{\x_i\}_{i\in[N]}$ will have negligible effect on the predictions of the classifier. To explain this fact, let $\mP_\perp$ denote the projection operator to the orthogonal complement of the data, so that $\mP_\perp \w$ is the component of $\w$ orthogonal to the data and
$\inp{\mP_\perp \w}{\vmu_c} = \inp*{\w}{\frac{\mP_\perp\vmu_c}{\norm{\mP_\perp\vmu_c}}} \norm{\mP_\perp\vmu_c}$.
Conditional on $(S_1,S_2)$ and the learning rule's random seed, the vector $\mP_\perp\vmu_c / \norm{\mP_\perp\vmu_c}$ is uniformly distributed on a unit sphere of dimension $d-N$ while the vector $\w$ is deterministic. Assuming without loss of generality that $\norm{\w}=1$, concentration of measure on the sphere implies that $\abs{\inp{\w}{\frac{\mP_\perp\vmu_c}{\norm{\mP_\perp\vmu_c}}}}$ is (with high probability) bounded by roughly $1/\sqrt{d}$, and therefore $\abs{\inp{\mP_\perp \w}{\vmu_c}}$ is roughly of the order $r_c / \sqrt{d}$. For sufficiently large $d$ (as required by the proposition), this inner product would be negligible, meaning that $\inp{\w}{\vmu_c}$ is roughly the same as $\inp{(I-\mP_\perp)\w}{\vmu_c}$, and $(I-\mP_\perp)\w$ is in the span of the data. The same argument applies to $\vmu_s$ as well.

In the second part of the proof, we consider classifiers of the form $\w = \sum_{i\in[N]} \beta_i y_i \x_i$ (which parameterizes the linear span of the data) and minimize $\inp{\w}{\vmu_s} / \inp{\w}{\vmu_c}$ over $\beta \in \R^N$ subject to the constraint that $\w$ has normalize margin of at least $\gamma$. To do so, we first use concentration of measure to argue that it is sufficient to lower bound $\sum_{i\in [N_1]}\beta_i$ subject to the margin constraint and $\norm{\w}^2 \le 1$, which is convex in $\beta$---we obtain this lower bound by analyzing the Lagrange dual of the problem of minimizing $\sum_{i\in [N_1]}\beta_i$ subject to these constraints.

Overall, we show a high-probability lower bound on $\frac{\inp{\w}{\vmu_s}}{\inp{\w}{\vmu_c}}$ that (for sufficiently high dimensions) scales roughly as $\frac{r_s^2 N_1 \gamma}{r_c^2 \sqrt{N_2}}$. For parameters satisfying \eqref{eq:negative_res_norm_constraint} we thus obtain $\frac{\inp{\w}{\vmu_s}}{\inp{\w}{\vmu_c}} \ge 1$, completing the proof.
\end{proof}

\paragraph{Implication for invariance-inducing algorithms.} Our proof implies that any interpolating algorithm should fail at learning invariant classifiers.
This alone does not necessarily imply that specific algorithms proposed in the literature for learning invariant classifiers fail, as they may not be interpolating.
Yet our simulations in \Cref{sec:validation} show that several popular algorithms proposed for eliminating spurious features are indeed interpolating in the overparameterized regime.
We also give a formal statement in \Cref{sec:irm_max_margin} regarding the IRMv1 penalty \citep{arjovsky2019invariant}, showing that it is biased toward large margins when applied to separable datasets.
Our results may seem discouraging for the development of invariance-inducing techniques using overparameterized models. It is natural to ask what type of methods \emph{can} provably learn such models, which is the topic of the next section.

\section{A Provably Invariant Overparameterized Estimator} \label{sec:positive_result}


We now turn to propose and analyze an algorithm (\Cref{alg:two_phase_learning}) that provably learns an overparametrized linear model with good robust accuracy in our setup. Our approach is 
a two-staged learning procedure that is conceptually similar to some recently proposed methods \citep{rosenfeld2022erm, veldanda2022fairness, kirichenko2022last,menon2021overparameterisation, pmlr-v180-kumar22a, pmlr-v162-zhang22u}.
In Section \ref{sec:validation} we validate our algorithm on simulations and on the Waterbirds dataset \cite{sagawa2019distributionally}, but we leave a thorough empirical evaluation of the techniques described here to future work. 

Let us describe the operation of \Cref{alg:two_phase_learning}. First, we evenly\footnote{The even split is used here for simplicity of exposition, and our full proof does not assume it. In practice, allocating more data to the first-stage split would likely perform better.} split the data from each environment into the sets $S^{\mathrm{train}}_e, S^{\mathrm{post}}_e$, for $e\in\{1,2\}$. The two stages of the algorithm operate on different splits of the data as follows.
\begin{enumerate}[leftmargin=*]
    \item \textbf{``Training" stage}: We use $\{S^{\mathrm{train}}_e\}$ to fit  overparameterized, interpolating classifiers $\{\w_e\}$ \emph{separately} for each environment $e\in{\{1,2\}}$.
    \item \textbf{``Post-processing" stage}: We use the second portion of the data $\left(S^{\mathrm{post}}_1,S^{\mathrm{post}}_2\right)$ to learn an invariant linear classifier over a new representation, which concatenates the outputs of the classifiers in the first stage. In particular, we learn this classifier by maximizing a score (i.e., minimizing an empirical loss), subject to an empirical version of an invariance constraint. For generality we denote this constraint by membership in some set of functions $\F(S^{\mathrm{post}}_1, S^{\mathrm{post}}_2)$.
\end{enumerate}

Crucially, the invariance penalty is only used in the second stage, in which we are no longer in the overparamterized regime since we are only fitting a two-dimensional classifier. In this way we overcome the negative result from Section \ref{sec:negative_result}. 

While our approach is general and can handle a variety of invariance notions (we discuss some of them in \Cref{sec:invariance_defs}), we analyze the algorithm under the Equal Opportunity (EOpp) criterion \citep{hardt2016equality}. Namely, for a model $f:\reals^d\rightarrow \reals$ we write:
\begin{align} \label{eq:equalized_odds}
    \F(S_1^{\mathrm{fine}}, S_2^{\mathrm{fine}}) = \big\{f : \hat{T}_1(f) &= \hat{T}_2(f) \big\}, \quad
    \text{where}\ \hat{T}_e(f) := \frac{4}{N_e}\sum_{(\x, y)\in{S_e^{\mathrm{fine}}}: y=1}{f(\x)}. \nonumber
\end{align}
This is the empirical version of the constraint $\E_{\sP_{\theta_1}} \left[ f(x) | y=1 \right]= \E_{\sP_{\theta_2}} \left[ f(x) | y=1 \right]$. From a fairness perspective (e.g., thinking of a loan application), this constraint ensures that the ``qualified'' members (i.e., those with $y=1$) of each group receive similar predictions, on average over the entire group.  


\begin{algorithm}[t]
\onehalfspacing
\caption{Two Phase Learning of Overparameterized Invariant Classifiers}
\label{alg:two_phase_learning}
\begin{algorithmic}
\Require Datasets $(S_1, S_2)$ 
and an invariance constraint function family $\F(\cdot, \cdot)$
\Ensure A classifier $f_{\rvv}(\x)$
\State Draw subsets of data without replacement $S^{\mathrm{train}}_e\subset S_e$ for $e\in{\{1,2\}}$ where $\left | S^{\mathrm{train}}_e \right| = N_e / 2$
\State Stage 1: 
Calculate 
$\w_e= 2N_e^{-1}\sum_{(\x, y)\in{S^{\mathrm{train}}_e}}{y\x}$ for each $e\in{\{1,2\}}$
\State Define $S_e^{\mathrm{fine}} = S_e \setminus S_e^{\mathrm{trn}}$ for $e\in{\{1, 2\}}$ and $S^{\mathrm{post}} = S_1^{\mathrm{fine}} \cup S_2^{\mathrm{fine}}$
\State Stage 2: Return
$f_{\rvv}(\x) = \inp{v_1\cdot\w_1 + v_2\cdot \w_2}{\x}$ 
that solves
\begin{equation*}
  \maximize \sum_{(\x,y)\in{S^{\mathrm{post}}}}{y f_{\rvv}(\x) } \quad\text{subject to}\quad {\norm{\rvv}_{\infty}=1} \quad\text{and}\quad  f_\rvv\in{\F(S_1^{\mathrm{fine}}, S_2^{\mathrm{fine}})}
\end{equation*}
\end{algorithmic}
\end{algorithm}




We now turn to providing conditions under which Algorithm~\ref{alg:two_phase_learning} successfully learns an invariant predictor.
The full proof for the following proposition can be found in \secref{sec:positive_proof_appendix} of the appendix.
While we do not consider the following proposition very surprising, the fact that it gives a finite sample learning guarantee means it does not directly follow from existing work (discussed in \S\ref{sec:discussion} below) that mostly assume inifinite sample size.

\begin{proposition}\label{prop:positive_result_main}
Consider the \learningsetting (\Cref{def:learning_setup}), and further suppose that $|\theta_1 - \theta_2| > 0.1$.\footnote{Intuitively, if $|\theta_1 - \theta_2|=0$ then the two training environments are indistinguishable and we cannot hope to identify that the correlation induced by $\vmu_s$ is spurious. Otherwise, we expect $|\theta_1 - \theta_2|$ to have a quantifiable effect on our ability to generalize robustly. For simplicity of this exposition we assume that the gap is bounded away from zero; the full result in the Appendix is stated in terms of $|\theta_1 - \theta_2|$.}
There exist universal constants $C_p,C_c, C_s \in (1, \infty)$ such that the following holds for every target robust error $\epsilon>0$ and failure probability $\delta\in(0,1)$. If $N_{\min} \defeq \min\{N_1,N_2\} \geq C_p\log(4 / \delta)$ for some $C_p\in(1,\infty),$\footnote{This assumption makes sure we have some positive labels in each environment.}  
\begin{align}
    &\normspu^2 \geq \ C_s \sqrt{\log\frac{68}{\delta}}\frac{\sigma^2\sqrt{d}}{N_{\min}}, ~ \normcore^2 \geq \ C_c\sigma^2\sqrt{\log\frac{68}{\delta}}\max\left\{Q^{-1}(\epsilon)\sqrt{\frac{d}{N_{\min}}}, \frac{\sqrt{d}}{N_{\min}}, \frac{\normspu^2}{N_{\min}\normcore^2} \right\}, \label{eq:mus_constraint_positive} \\ 
    &\mbox{and} ~~ d \geq \log \frac{68}{\delta} \label{eq:muc_constraint_positive}
\end{align}
then, with probability at least $1-\delta$ over the draw of the training data and the split of the data between the two stages of learning, the robust error of the model returned by Algorithm~\ref{alg:two_phase_learning} does not exceed $\epsilon$.
\end{proposition}

\begin{proof}[Proof sketch]
Writing down the error of $f_{\rvv} = v_1 \cdot \w_1 + v_2 \cdot \w_2$ under $\P_\theta$, it can be shown that to obtain the desired bound on the robust error of the classifier returned by Algorithm~\ref{alg:two_phase_learning}, we must upper bound the ratio
\begin{equation*}
    \frac{(v_1^\star\theta_1 + v_2^\star\theta_2)\norm{\vmu_s}^2 + \inp{\vmu_s}{v_1^\star\bar{\rvn}_1 + v_2^\star\bar{\rvn}_2}}{(v_1^\star + v_2^\star)\norm{\vmu_c}^2 + \inp{\vmu_c}{v_1^\star\bar{\rvn}_1 + v_2^\star\bar{\rvn}_2}},
\end{equation*}
when $\bar{\rvn}_e$ is the mean of Gaussian noise vectors, and $v_1^\star$ and $v_2^\star$ are the solutions to the optimization problem in Stage 2 of \Cref{alg:two_phase_learning}. The terms involving inner-products with the noise terms are zero-mean and can be bounded using standard Gaussian concentration arguments. Therefore, the main effort of the proof is upper bounding
\begin{equation*}
    \frac{v_1^\star\theta_1 + v_2^\star\theta_2}{v_1^\star + v_2^\star} \cdot \frac{\norm{\vmu_s}^2}{\norm{\vmu_c}^2}.
\end{equation*}
To this end, we leverage the EOpp constraint. The population version of this constraint (corresponding to infinite $N_1$ and $N_2$) implies that $v_1^\star\theta_1 + v_2^\star\theta_2 = 0$. For finite sample sizes, we use standard Gaussian concentration and the Hanson-Wright inequality to show that the empirical EOpp constraint implies that $| v_1^\star\theta_1 + v_2^\star\theta_2 |$ goes to zero as the sample sizes increase. Furthermore, we argue that  $\left| v_1^\star + v_2^\star \right| \geq |\theta_1-\theta_2|/2$, implying that---for appropriately large sample sizes---the above ratio indeed goes to zero. 
\end{proof}

\section{Empirical Validation}
\label{sec:validation}

\begin{figure}
    \centering
    \vspace{-15pt}
    \includegraphics[width=1.\linewidth]{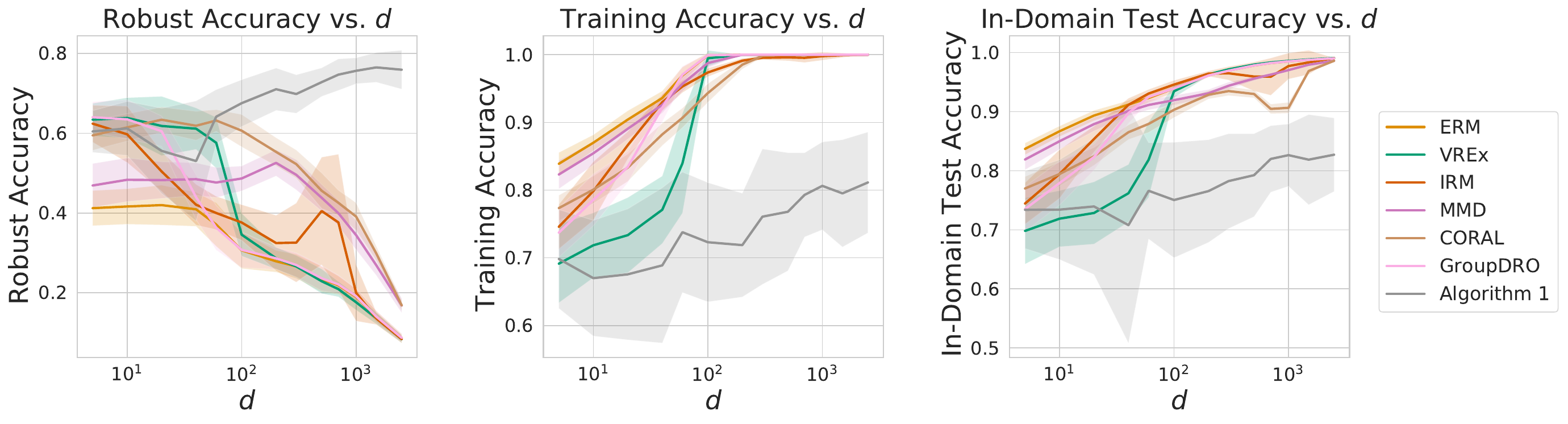}
    \vspace{-17pt}
    \caption{Numerical validation of our theoretical claims. Invariance inducing methods improve robust accuracy compared to ERM in low values of $d$, but their ability to do so is diminished as $d$ grows (top plot) and they enter the interpolation regime, as seen on the bottom plot for $d>10^2$. \Cref{alg:two_phase_learning} learns robust predictors as $d$ grows and does not interpolate.}
    \label{fig:simulation}
    \vspace{-12pt}
\end{figure}

The empirical observations that motivated this work can be found across the literature.
\yc{add references or point to the place in the paper where we provide them}
We therefore focus our simulations on validating the theoretical results in our simplified model. We also evaluate \Cref{alg:two_phase_learning} on the Waterbirds dataset, where the goal is not to show state-of-the-art results, but rather to observe whether our claims hold beyond 
the \learningsetting. 
\subsection{Simluations} \label{sec:simualtion}
\paragraph{Setup.} We generate data as described in \thmref{thm:main_statement} with two environments where $\theta_1=1, \theta_2=0$ (see \Cref{fig:simulation_neg_theta} in the appendix for results of the same simulation when $\theta=-\frac{1}{2}$). We further fix $r_c=1$ and $r_c=2$, while $N_1=800$ and $N_2=100$. We then take growing values of $d$, while adjusting $\sigma$ so that $\left(r_c / \sigma\right)^2 \propto \sqrt{d / N}$.\footnote{This is to keep our parameters within the regime where benign overfitting occurs.} For each value of $d$ we train linear models with IRMv1 \citep{arjovsky2019invariant}, VREx \citep{krueger2020out}, MMD \citep{li2018deep}, CORAL \citep{sun2016deep}, GroupDRO \citep{sagawa2019distributionally},  implemented in the Domainbed package \citep{gulrajani2020search}. We also train a classifier with the logistic loss to minimize empirical error (ERM), and apply \Cref{alg:two_phase_learning} where the ``post-processing" stage trains a linear model over the two-dimensional representation using the VREx penalty to induce invariance. We repeat this for $15$ random seeds for drawing $\vmu_c, \vmu_s$ and the training set.

\paragraph{Evaluation and results.} We compare the robust accuracy and the train set accuracy of the learned classifiers as $d$ grows.
First, we observe that all methods except for \Cref{alg:two_phase_learning} attain perfect accuracy for large enough $d$, i.e., they interpolate.
 We further note that while invariance-inducing methods give a desirable effect in low dimensions (the non-interpolating regime)---significantly improving the robust error over ERM---they become aligned with ERM in terms of robust accuracy as they go deeper into the interpolation regime (indeed, IRM essentially coincides with ERM for larger $d$). This is an expected outcome considering our findings in \secref{sec:negative_result}, as we set here $N_1$ to be considerably larger than $N_2$. \yc{would we get nicer results if we picked $N_1/N_2$ to be even larger?} 
\subsection{Waterbirds Dataset}

We evaluate Algorithm \ref{alg:two_phase_learning} on the Waterbirds dataset \citep{sagawa2019distributionally}, which has been previously used to evaluate the fairness and robustness of deep learning models. \yc{todo: add additional references supporting this claim}

\paragraph{Setup.} Waterbirds is a synthetically created dataset containing images of water- and land-birds overlaid on water and land background. Most of the waterbirds (landbirds) appear in water (land) backgrounds, with a smaller minority of waterbirds (landbirds) appearing on land (water) backgrounds. We set up the problem following previous work \citep{pmlr-v119-sagawa20a, veldanda2022fairness}, where a logistic regression model is trained over random features extract from a fixed pre-trained ResNet-18. Please see \Cref{sec:waterbirds_details} for details.
\paragraph{Fairness.} We use the image background type (water or land) as the sensitive feature, denoted $A$, and consider the fairness desiderata of Equal Opportunity \cite{hardt2016equality}, i.e., the false negative rate (FNR) should be similar for both groups. Towards this, we use the MinDiff penalty term \citep{prost2019toward}. The
\paragraph{Evaluation.} We compare the following methods: \textbf{(1) Baseline}: Learning a linear classifier $\w$ by minimizing $\mathcal{L}_p + \lambda \cdot \mathcal{L}_M$, where $\mathcal{L}_p$ is the standard binary cross entropy loss and $\mathcal{L}_M$ is the MinDiff penalty; \textbf{(2) Algorithm \ref{alg:two_phase_learning}}: In the first stage, we learn group-specific linear classifiers $\w_0, \w_1$ by minimizing $\mathcal{L}_p$ on the examples from $A=0$ and $A=1$, respectively. In the second stage we learn $\vv \in \R^2$ by minimizing $\mathcal{L}_p + \lambda \cdot \mathcal{L}_M$ on examples the entire dataset, where the new representation of the data is $\tilde{X} = [\langle w_1, X\rangle, \langle w_2, X\rangle] \in \R^2$.\footnote{This is basically Algorithm \ref{alg:two_phase_learning} with the following minor modifications: (1) The $\w_e$'s are computed via ERM, rather than simply taken to be the mean estimators; (2) Since the FNR gap penalty is already computed w.r.t.\ a small number of samples, we avoid splitting the data and use the entire training set for both phases; (3) we convert the constrained optimization problem into an unconstrained problem with a penalty term.} 

\paragraph{Results.}
Our main objective is to understand the effect of the fairness penalty. Toward this, for each method we compare both the test error and the test FNR gap when using either $\lambda=0$ (no regularization) or $\lambda=5$. The results are summarized in Figure \ref{fig:waterbirds}. We can see that for the baseline approach, the fairness penalty successfully reduces the FNR gap when the classifier is not interpolating. However, as our negative result predicts and as previously reported in \cite{veldanda2022fairness}, the fairness penalty becomes ineffective in the interpolating regime ($d \geq 1000$). On the other hand, for our two-phased algorithm, the addition of the fairness penalty reduces does reduce the FNR gap with an average relative improvement of 20\%; crucially, this improvement is
 independent of $d$.

\begin{figure}
    \centering
    \includegraphics[width=0.75\linewidth]{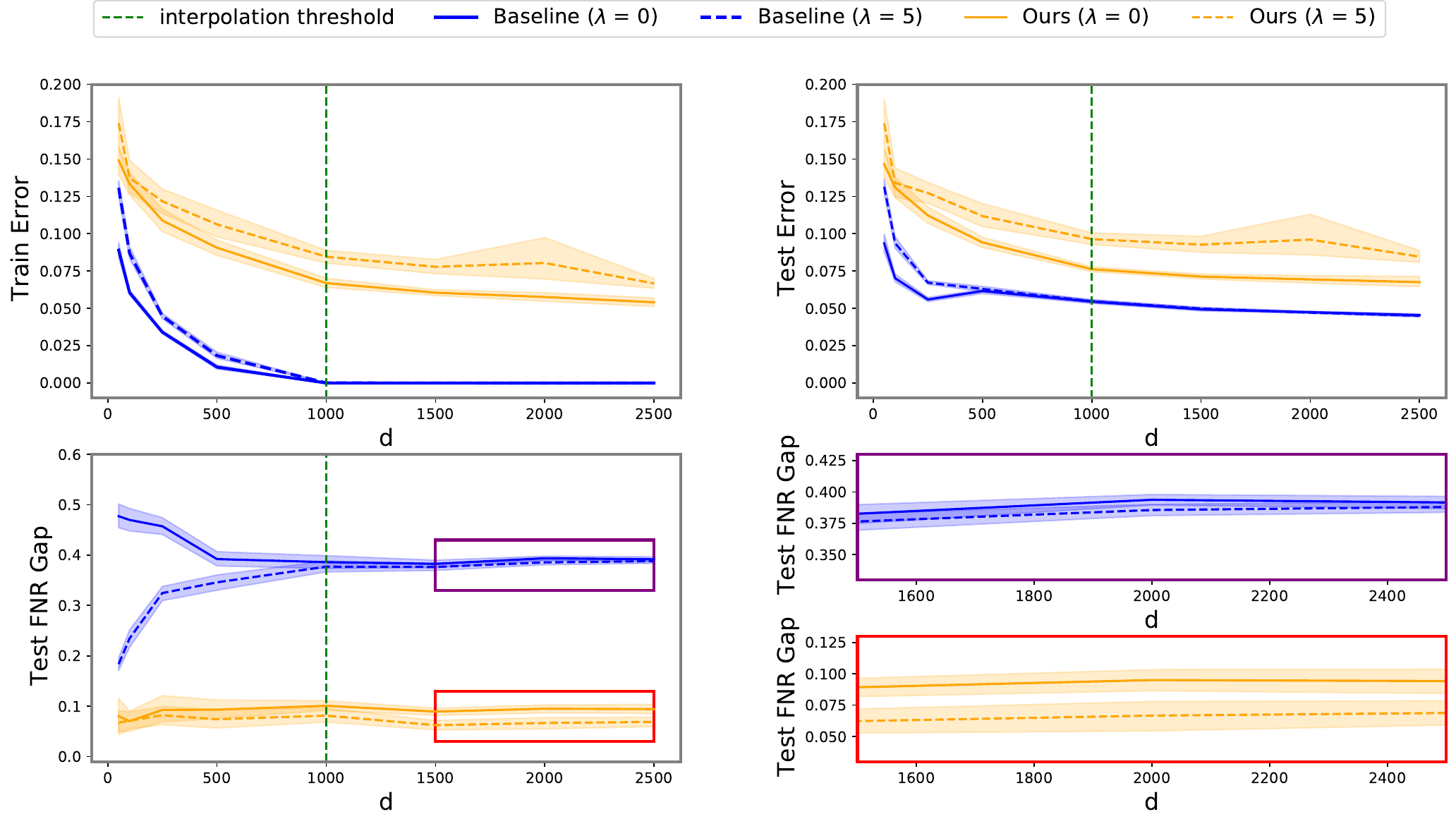}
    \vspace{-10pt}
    \caption{Results for the Waterbirds dataset \citep{sagawa2019distributionally}. \textbf{Top row}: Train error (left) and test error (right). The train error is used to identify the interpolation threshold for the baseline method (approximately $d=1000$). \textbf{Bottom row}: Comparing the FNR gap on the test set (left), with zoomed-in versions on the right. 
    }
    \label{fig:waterbirds}
    \vspace{-10pt}
\end{figure}

\section{Discussion and Additional Related Work} \label{sec:discussion}
In terms of formal results, most existing guarantees about invariant learning algorithms rely on the assumption that infinite training data is available \citep{arjovsky2019invariant, wald2021calibration, veitch2021counterfactual, puli2021out, rosenfeld2020risks, https://doi.org/10.48550/arxiv.2110.12403}. \citet{pmlr-v162-wang22x, chen2022iterative} analyze algorithms that bear resemblance to \Cref{alg:two_phase_learning} as they first project the data to a lower dimension and then fit a classifier.
While these algorithms deal with more general assumptions in terms of the number of environments, number of spurious features, and noise distribution, the fact that their guarantees assume infinite data prevents them from being directly applicable to \Cref{alg:two_phase_learning}. 
%
A few works with results on finite data are \citet{ahuja2021empirical, https://doi.org/10.48550/arxiv.2205.15196} (and also \citet{pmlr-v178-efroni22a} who work on related problems in the context of sequential decision making) that characterize the sample complexity of methods that learn invariant classifiers. However, they do not analyze the overparameterized cases we are concerned with.

Negative results about learning overparameterized robust classifiers have been shown for methods based on importance weighting \citep{zhai2022understanding} and max-margin classifiers \citep{pmlr-v119-sagawa20a}.
Our result is more general, applying to any learning algorithm that separates the data with arbitrarily small margins, instead of focusing on max-margin classifiers or specific algorithms. While we focus on the linear case, we believe it is instructive, as any reasonable method is expected to succeed in that case. Nonetheless, we believe our results can be extended to non-linear classifiers, and we leave this to future work.

One take-away from our result is that while low training loss is generally desirable, overfitting to the point of interpolation can significantly hinder invariance-inducing objectives. This means one cannot assume a typical deep learning model with an added invariance penalty will indeed achieve any form of invariance; this fact also motivates using held-out data for imposing invariance, as in our \Cref{alg:two_phase_learning} as well as several other two-stage approaches mentioned above. 


Our work focuses theory underlying a wide array of algorithms, and there are natural follow-up topics to explore. One is to conduct a comprehensive empirical comparison of two-stage methods along with other methods that avoid interpolation, e.g., by subsampling data \citep{idrissi2022simple, chatterji2022under}. Another interesting topic is whether there are other model properties that are incompatible with interpolation. For instance, recent work \citep{https://doi.org/10.48550/arxiv.2210.01964} connects the generalization gap and calibration error on the training distribution.
We also note that our focus in this paper was not on types of invariance that are satisfiable by using clever data augmentation techniques (e.g. invariance to image translation), or the design of special architectures (e.g. \cite{cohen2016group,lee2019set,maron2019universality}). These methods carefully incorporate a-priori known invariances, and their  empirical success when applied to large models may suggest that there are lessons to be learned for the type of invariant learning considered in our paper. These connections  seem like an exciting avenue for future research.

\yc{todo (maybe): make a paragraph dedicated to the recent prior work that explicitly sacrifice interpolation in search of invariance. this will be a good place to emphasize that methods that subsample the training set and then interpolate only the subsample are \emph{not} interpolating methods, and therefore can in principle be invariant}


\bibliography{malign_overfitting}
\bibliographystyle{malign_overfitting}

\newpage

\appendix
\newpage
\section{Setting and Helper Lemmas}\label{sec:app_helpers}
\subsection{Notation} \label{sec:notation}
Let $\mathbb{U}(\mathrm{O}(d))$ be the uniform distribution over $d\times d$ orthogonal matrices, $\mathrm{Rad}(\alpha)$ the Rademacher distribution with parameter $\alpha$, and $\N(\vmu, \Sigma)$ the Gaussian and multivariate normal distribution with mean $\vmu$ and covariance $\Sigma$ (the dimension will be clear from context) and $W(\Sigma, d)$ the Wishart distribution with scale matrix $\Sigma$ and $d$ degrees of freedom. For the dataset $S=\{\x_i, y_i\}_{i=1}^{N}$ we denote the indices of examples with set $\idxset=[N]$, and recalling that $S$ comprises two datasets $S_1, S_2$, we denote the indices of their respective examples within $S$ by $\idxset_1, \idxset_2\subseteq{\idxset}$ where $|\idxset_e|=N_e$ for $e\in{\{1,2\}}$. Our generative process is then: 
\begin{align*}
\rmU &\sim \mathbb{U}(\mathrm{O}(d)) \\
\vmu_c &= U_1\cdot r_c, \vmu_s = U_2 \cdot r_s \\
y_i &= \mathrm{Rad}(\half), n_i\sim \mathcal{N}(0, \sigma^2\I_d) \quad \forall i\in{[N]}\\
\x_i &= y_i\vmu_c + y_i\theta_e\vmu_s + n_i \quad \forall e, i\in{\idxset_e}.
\end{align*}

\begin{figure}
    \centering
    \includegraphics[width=0.75\linewidth]{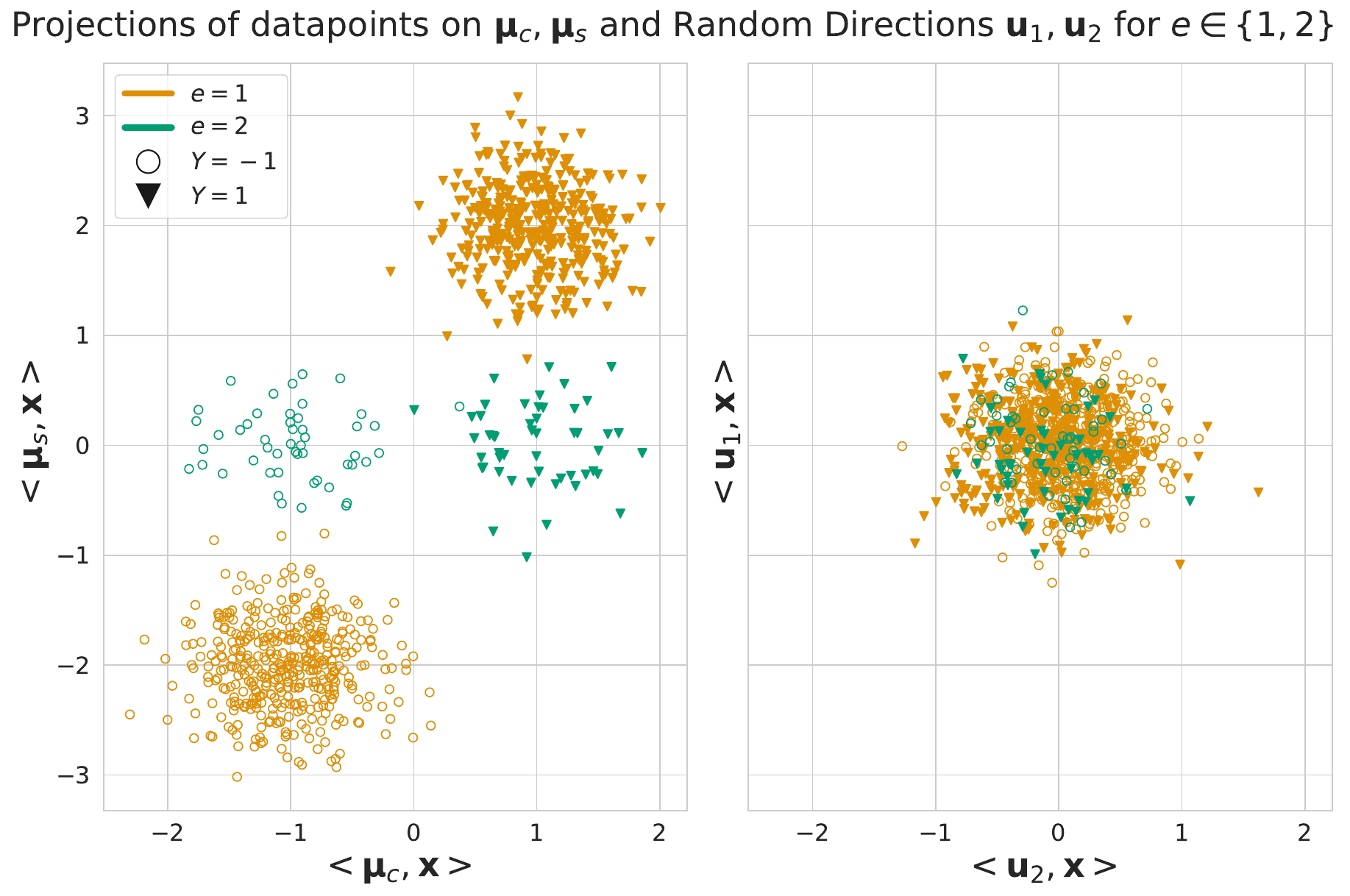}
    \caption{Example of datasets sampled from two training envrionments, where we set $\theta_1=1, \theta_2=0,$ $N_1=800, N_2=100, r_s=2, r_c=1$. Left and right plots show projections of training points on $\vmu_c, \vmu_s\in{\reals^d}$, and on $\rvu_1, \rvu_2$ drawn uniformly from the $d$-dimensional unit sphere, respectively. As we increase $d$ there are many hyperplanes $\w\in{\reals^d}$ that separate the data, for some $\inp{\w}{\mu_c}$ is much higher than $\inp{\w}{\mu_s}$ (i.e. their predictions are invariant) and for some the opposite may hold. We ask whether interpolating learning rules can find the former.}
    \label{fig:projections}
\end{figure}
\begin{figure}
    \centering
    \includegraphics[width=0.75\linewidth]{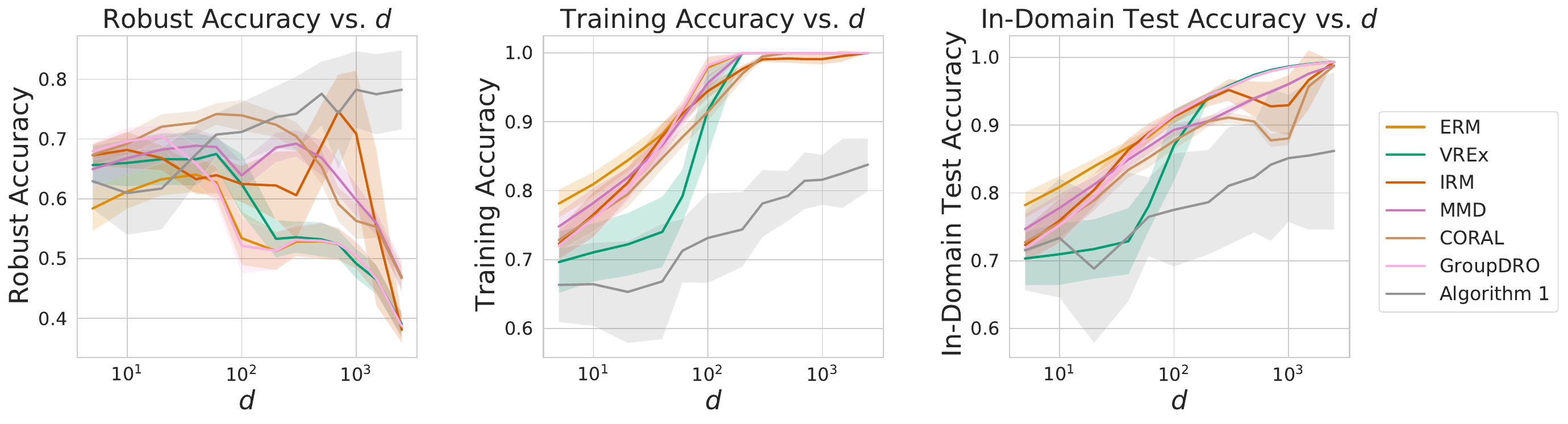}
    \caption{Results of the simulation described in \Cref{sec:simualtion} with $\theta_2=-\frac{1}{2}$ (all other parameters are kept at the same value as in \Cref{sec:simualtion})}.
    \label{fig:simulation_neg_theta}
\end{figure}

\newcommand{\wperp}{\w_\perp}

The vectors $E_1, E_2\in{\{0,1\}^N}$ are binary vectors where $[E_e]_i = 1$ for $i\in{\idxset_e}$ and $e\in{\{1,2\}}$, while $\mathbf{1}$ is the vector of length $N$ whose entries equal $1$. We also denote $\z_i = \x_i y_i$ for $i\in{\idxset}$ and $\rmZ=[\z_1, \ldots, \z_N]^\top\in{\reals^{N\times d}}$ the matrix that stacks all these vectors.
The $i$-th column of a matrix $\rmM$ is denoted by $M_i$, $s_{\min}(\rmM), s_{\max}(\rmM)$ are its smallest and largest singular values accordingly.
The unit matrix of size $n$ is denoted by $\I_n$ and for convenience we denote the direction of any vector $\bv$ as $\hat{\bv}:=\frac{\bv}{\|\bv\|}$. Finally, for some vector of coefficients $\mBeta\in{\reals^N}$, we will use the form $\hat{\w}=\sum_{i\in{\idxset}}{\beta_i y_i \x_i} + \w_{\perp}$ where $\wperp$ is in the orthogonal complement of ${\mathrm{span}(\{\x_i\}_{i\in{\idxset}})}$, to write any linear model (here normalized to unit norm).

For convenience we will write our proofs for the case where $\theta_1=1$ and $\sigma^2=d^{-1}$, extensions to different settings of these parameters are straightforward but result in a more cumbersome notation.

\yc{missing: high-level outline of proof, i.e., explain organization of this section}

\subsection{Operator Norms of Wishart Matrices} \label{sec:spectral_bounds}
We begin with stating the required events for our results and their occurrence with high-probability:
\begin{lemma} \label{lem:high_prob_events}
Consider the matrix $\rmG = \rmZ - \ones\vmu_c^\top - (E_1 + \theta_2E_2)\vmu_s^\top$. \yc{todo: refactor $E_1$ to something consistent with lowercase boldface vector convention. Also: define $\rmG$, $\rmZ$, $E_1$, in displayed equations for better visibility} For any $t > 0$, with probability at least $1-10\exp(-t^2 / 8)$ the following hold simultaneously:
\begin{align}
    1-\sqrt{\frac{N}{d}}-\frac{t}{\sqrt{d}} \leq s_{\min}(\G^\top) &\leq s_{\max}(\G^\top) \leq 1+\sqrt{\frac{N}{d}}+\frac{t}{\sqrt{d}} \label{eq:random_singvals_bound} \\
    \norm{\G\vmu_c} &\leq t \sqrt{\frac{N}{d}}\norm{\vmu_c} \label{eq:G_muc_bound} \\
    \norm{\G\vmu_s} &\leq t \sqrt{\frac{N}{d}}\norm{\vmu_s} \label{eq:G_mus_bound}
\end{align}
\end{lemma}
\begin{proof}
$\G$ is a random Gaussian matrix with $G_{i,j}\sim \mathcal{N}(0, d^{-1} \I_N)$. By concentration results for random Gaussian matrices \cite[Cor. 5.35]{vershynin_2012} we obtain that with probability at least $1-2\exp(-t^2 / 2) \geq 1-2\exp(-t^2 / 8)$ \eqref{eq:random_singvals_bound} holds.

Next we note that $\G\vmu_c\sim\N (0, d^{-1}\|\vmu_c\|^2 \I_N)$ and similarly for $\G\vmu_s$. The norm of a Gaussian random vector can be bounded for any $t_2 > 0$ \citep[Eq. 3.4]{ledoux2013probability}: \yc{ref?}
\begin{align*} 
    \sP \left[ \| \G\vmu_c \| \geq t_2 \right] \leq 4\exp\left( -\frac{d t_2^2}{8N\norm{\vmu_c}^2} \right)
\end{align*}
Setting $t_2=t\sqrt{\frac{N}{d}}\norm{\vmu_c}$ we get that with probability at least $1-4\exp(-t^2 / 8)$ \eqref{eq:G_muc_bound} holds. Repeating the analogous derivation for \eqref{eq:G_mus_bound} and taking a union bound over the $3$ events, we arrive at the desired result.
\end{proof}

\begin{lemma} \label{lem:cov_spectral_bound}
Conditioned on the events in \lemref{lem:high_prob_events} with parameter $t\ge 0$,
if 
\begin{align} \label{eq:assum_spectral}
     \frac{\sqrt{N}+t}{\sqrt{d}} + \sqrt{N}(\norm{\vmu_c} + \norm{\vmu_s}) \leq \half,
\end{align}
then
\begin{align*}
    \norm{\rmZ\rmZ^\top - \E[\rmZ\rmZ^\top]}_{\mathrm{op}} \leq 3\frac{\sqrt{N}+t}{\sqrt{d}}
    ~~\mbox{and}~~
    \half I_N \preceq \rmZ\rmZ^\top \preceq 2 I_N.
\end{align*}
\end{lemma}
We note that we already assume $d \gg N$ and $\norm{\vmu_c} \ll N^{-1/2}$, hence the additional assumption introduced in the conditions of this lemma is regarding the size of $\norm{\vmu_s}\sqrt{N_1}$.
\begin{proof}
Since $\G\G^\top \sim W(d^{-1}\I_N, d)$ we have that $\E [\G \G^\top] = \I_N$. 
Then from \eqref{eq:random_singvals_bound} we can also obtain $(1 - \sqrt{\frac{N}{d}} - \frac{t}{\sqrt{d}})^2 \, \I_n \preceq \G \G^\top \preceq (1 + \sqrt{\frac{N}{d}} + \frac{t}{\sqrt{d}})^2 \, \I_n$,
which leads to:
\begin{align*}
    \left\|\G \G^\top - \E[\G \G^\top]\right\|_{\mathrm{op}} \leq \left(1+\sqrt{\frac{N}{d}} + \frac{t}{\sqrt{d}}\right)^2 - 1.
\end{align*}

Combining this with \eqref{eq:G_muc_bound} and \eqref{eq:G_mus_bound}
\begin{align*}
    \norm{\rmZ\rmZ^\top - \E\left[ \rmZ \rmZ^\top \right]}_{\mathrm{op}} &\leq \norm{\G\G^\top - \E\left[ \G\G^\top\right]}_{\mathrm{op}} + \norm{\G\vmu_c\ones^\top}_{\mathrm{op}} + \norm{\G\vmu_s\left(E_1 + \theta_2 E_2\right)^\top}_{\mathrm{op}} \\
    &\leq \sqrt{\frac{N}{d}}\left(2\frac{\sqrt{N}+t}{\sqrt{N}} + \frac{(\sqrt{N}+t)^2}{\sqrt{Nd}} + t\sqrt{N}(\norm{\vmu_c} + \norm{\vmu_s})\right) \\
    &\leq \frac{\sqrt{N}+t}{\sqrt{d}}\left(2 +  
    \frac{\sqrt{N}+t}{\sqrt{d}} + \frac{t}{\sqrt{N}+t}\sqrt{N}(\norm{\vmu_c} + \norm{\vmu_s}) \right)
    \\ & \leq\frac{\sqrt{N}+t}{\sqrt{d}} \cdot 2.5,
\end{align*}
where the last transition follows from substituting \eqref{eq:assum_spectral}.
To obtain the spectral bound on $\rmZ\rmZ^\top$ we have that $\rmZ=\G+\ones\vmu_c^\top + (E_1 + \theta_2 E_2)\vmu_s^\top$.
From Weyl's inequality for singular values: 
\begin{align*}|s_{\mathrm{min}}(\G^\top + \vmu_c\ones^\top +& \vmu_s \left(E_1 + \theta_2 E_2\right)^\top) - s_{\mathrm{min}}(\G^\top)| \leq\\
& s_{\mathrm{max}}( \vmu_c\ones^\top + \vmu_s \left(E_1 + \theta_2 E_2\right)^\top)
\leq \norm{\vmu_c}\sqrt{N} + \norm{\vmu_s}\sqrt{N}.
\end{align*}

Taken together with \eqref{eq:random_singvals_bound} and the assumption in \eqref{eq:assum_spectral} we get:
\begin{align*}
    s_{\mathrm{min}}(\rmZ^\top) &\geq s_{\mathrm{min}}(\G^\top) - \norm{\vmu_c}\sqrt{N} - \norm{\vmu_s}\sqrt{N} \\
    &\geq 1-\frac{1}{\sqrt{d}}\left(\sqrt{N} + t\right) - \norm{\vmu_c}\sqrt{N} - \norm{\vmu_s}\sqrt{N} \\
    &\geq \half.
\end{align*}
To prove that $\rmZ\rmZ^\top \preceq 2$ we simply need to follow the same steps while taking notice that Weyl's inequality also holds for $s_{\mathrm{max}}(\G^\top)$. This will give us $s_{\mathrm{max}}(\rmZ^\top) \leq 3 / 2 \leq 2$ from which the upper bound follows.
\end{proof}


\newcommand{\vmuperp}{\vmu^\perp}
\newcommand{\normalize}[1]{\frac{#1}{\norm{#1}}}

\subsection{Sufficiency of Linear Classifiers Spanned by Data Points}
\yc{I think it makes more sense to put this subsection in the next appendix, since this is part of the proof of the negative result}
\yc{deterministic algorithm assumption not necessary and not part of current statement - fix this}
We wish to bound $\langle \hat{\w}_\bot , \vmu_c \rangle= r_c \langle \hat{\w}_\perp, U_1 \rangle$. To this end let us take an orthonormal basis $\{\bv_1,  \ldots, \bv_N\}$ and let these vectors form the columns of the orthogonal matrix $V\in\R^{d\times N}$. \yc{to do: correct matrix notation for $V$; probably $\rmV$?} 
Let $P_V$ be the orthogonal projection matrix on the columns of $V$.
We first claim that conditioned on the data, the component of the mean vectors that is not spanned by the data is distributed uniformly.
\begin{lemma} \label{lem:uniform_ortho_mean}
Let $\vmuperp_{c} \coloneqq (I-P_V)\vmu_c$ and $\vmuperp_{s}\coloneqq(I-P_V)\vmu_c$. \yc{$P_V$ should probably $\rmP_{\rmV}$? We never actually make any meaningful use of the $\rmV$ notation, so we might want to drop it}Conditional on the training set $\{\x_i, y_i\}_{i\in{S}}$, \yc{and the algorithm's randomness} the vectors $\normalize{\vmuperp_s}$ and $\normalize{\vmuperp_c}$ are uniformly distributed on unit spheres a subspace of dimension $d-N$. 
\end{lemma}
\begin{proof}
Recalling the notation $\z_i = y_i \x_i$, note that $\{\z_i\}_{i\in \idxset}$ are sufficient statistics for $\vmu_s, \vmu_c$ given the training data, i.e., $\P(\vmu_s, \vmu_c \mid \{\z_i\}_{i\in \idxset})=\P(\vmu_s, \vmu_c \mid \{\x_i, y_i\}_{i\in \idxset})$. Furthermore, since the joint distribution of $\vmu_s, \vmu_c, \{\z_i\}_{i\in \idxset}$ is rotationally invariant, we have
\begin{equation*}
    \P(\vmu_s, \vmu_c \mid \{\z_i\}_{i\in \idxset}) = \P(\rmR\vmu_s, \rmR\vmu_c \mid \{\rmR\z_i\}_{i\in \idxset})
\end{equation*}
for any orthogonal matrix $\rmR\in\R^{d\times d}$. Focusing on matrices $\rmR$ that presereve that data, i.e., satisfying $\rmR\z_i = \z_i$ for all $i\in [N]$, we have
\begin{equation*}
    \P(\vmu_s, \vmu_c \mid \{\z_i\}_{i\in \idxset}) = \P(\rmR\vmu_s, \rmR\vmu_c \mid \{\z_i\}_{i\in \idxset}).
\end{equation*}
We may also write this equality as
\begin{flalign*}
    &\P(P_V\vmu_s, P_V\vmu_c, (I-P_V)\vmu_s, (I-P_V)\vmu_c \mid \{\z_i\}_{i\in \idxset}) \\ & \hspace{16pt}= \P(P_V\rmR\vmu_s, P_V \rmR \vmu_c, (I-P_V)\rmR\vmu_s, (I-P_V)\rmR\vmu_c \mid \{\z_i\}_{i\in \idxset}).
\end{flalign*}
The fact that $R$ preserves $\{\z_i\}_{i\in \idxset}$ implies that $P_V \rmR = P_V = \rmR P_V$ and therefore
\begin{equation*}
    \P(P_V\vmu_s, P_V\vmu_c, \vmuperp_s, \vmuperp_c \mid \{\z_i\}_{i\in \idxset}) = \P(P_V\vmu_s, P_V  \vmu_c, \rmR \vmuperp_s,\rmR\vmuperp_c \mid \{\z_i\}_{i\in \idxset}).
\end{equation*}
Marginalizing $P_V\vmu_s, P_V\vmu_c$, we obtain that, conditional on the training data, the distribution of $\vmuperp_s, \vmuperp_c$, is invariant to rotations that preserve the training data. Therefore, the unit vectors in the directions of $\vmuperp_s$  and $\vmuperp_c$ must each be uniformly distributed on the sphere orthogonal to the training data, which has dimension $d-N$. 
\end{proof}
Now we simply need to derive a bound on $\langle\wperp, \vmu_s \rangle$:
\begin{corollary} \label{coro:ortho_bound}
For any $t>0$ as in \lemref{lem:high_prob_events}, with probability at least $1-14\exp(-t^2 / 8)$, all the events in \lemref{lem:high_prob_events} hold and additionally
\begin{align} \label{eq:data_span_bound}
    \left| \langle\wperp, \vmu_s \rangle \right| < \frac{\norm{\vmu_s}}{\sqrt{d-N}}t
    ~~\mbox{and}~~
    \left| \langle\wperp, \vmu_c \rangle \right| < \frac{\norm{\vmu_c}}{\sqrt{d-N}}t.
\end{align}
\end{corollary}
\begin{proof}
Note that
\begin{equation*}
    \left| \langle\wperp, \vmu_s \rangle \right|
    =
    \left| \langle\wperp, \vmuperp_s \rangle \right|
    =
    \norm{\vmuperp_s}\norm{\wperp}\left| \left\langle \normalize{\wperp}, \normalize{\vmuperp_s} \right\rangle \right|
    \le 
    \norm{\vmu_s}\left| \left\langle \normalize{\wperp}, \normalize{\vmuperp_s} \right\rangle \right|.
\end{equation*}
Conditional on the training data and the algorithm's randomness, $\normalize{\wperp}$ is a fixed unit vector in the subspace orthogonal to the training data (of dimension $d-N$), while $\normalize{\vmuperp_s}$ is a spherically uniform unit vector in that subspace. Therefore, standard concentration bounds \cite[Lemma~2.2]{ball1997elementary} imply that, for any $t_2>0$ 
\begin{align*}
    \mathbb{P}\left( \left| \left\langle \normalize{\wperp}, \normalize{\vmuperp_s} \right\rangle \right| \ge t_2 \right) \le 2\exp(-(d-N) t_2^2 /2).
\end{align*}
The claimed result follows by taking $t_2 = t/\sqrt{d-N}$, applying the same argument for $\vmu_c$ and taking a union bound.
\end{proof}


\section{Proofs of Main Result}
In this section, we provide the proof of \Cref{prop:neg_result_details}, our main theoretical finding highlighting a fundamental limitation to the robustness of any interpolating classifier. Following the notation of \Cref{sec:app_helpers}, we write a general unit-vector classifier as $\hat{\w} = \sum_{i\in S} \beta_i \z_i + \wperp$, where $\z_i = y_i \x_i$. As explained in the proof sketch at \Cref{sec:negative_result}, in order to show a lower bound on robust accuracy, we show a lower bound on the spurious-to-core ratio  $\frac{\inp{\w}{\vmu_s}}{\inp{\w}{\vmu_c}}$ or equivalently upper bound $\frac{\inp{\w}{\vmu_s}}{\inp{\w}{\vmu_c}}$, which we can write as
\begin{align}\label{eq:core_spu_ratio_expans}
\frac{\langle \w, \vmu_c\rangle}{\langle \w, \vmu_s\rangle}
= \frac{\langle \hat{\w}, \vmu_c\rangle}{\langle \hat{\w}, \vmu_s\rangle} 
= \frac{\norm{\vmu_c}^2}{\norm{\vmu_s}^2} \cdot 
\frac{\ones^\top\beta + \frac{1}{\norm{\vmu_c}^2}\left[\sum_{i\in{\idxset}}{\beta_i\langle n_i, \vmu_c \rangle +\langle \wperp, \vmu_c \rangle}\right]}{(E_1 + \theta_2 E_2)^\top\beta + \frac{1}{\norm{\vmu_s}^2}\left[\sum_{i\in{\idxset}}{\beta_i\langle n_i, \vmu_s \rangle} + \langle \wperp, \vmu_s \rangle\right]}.
\end{align}

We develop the lower bound - and prove \Cref{prop:neg_result_details} - in three steps, each corresponsding to a subsection below. First, we give a lower bound on $(E_1 + \theta_2E_2)^\top \mBeta$ using Lagrange duality (\Cref{lem:spurious_inner_prod_bnd}). Second, in \Cref{lem:residual_upper_bnd}, we bound the residual terms of the form $\frac{1}{\norm{\vmu}^2}\left|\sum_{i\in{\idxset}}{\beta_i \langle n_i, \vmu \rangle} + \langle \wperp, \vmu \rangle\right|$ (for $\vmu \in \{\vmu_c,\vmu_s\}$) using concentration of measure arguments from \Cref{sec:app_helpers}. Finally, we combine these two results with the conditions of \Cref{prop:neg_result_details} to conclude its proof.

\subsection{Lower bounding \texorpdfstring{$(E_1 + \theta_2 E_2)^\top \mBeta$}{}}

The crux of our proof is showing that the term $(E_1 + \theta_2E_2)^\top \mBeta$, i.e., the sum of the contributions of elements from the first environment to $\w$, must grow roughly as $N_1\gamma$ for any interpolating classifier. This will in turn imply a large spurious component in the classifier via manipulation of \eqref{eq:core_spu_ratio_expans}.

\begin{lemma}\label{lem:spurious_inner_prod_bnd}
Conditional on the events in \corref{coro:ortho_bound} (with parameter $t > 0$), if \eqref{eq:assum_spectral} holds and $\w$ has normalized margin at least $\gamma$, we have that
\begin{align}\label{eq:key-negative-bound}
(E_1 + \theta_2E_2)^\top\mBeta \geq \frac{1}{2}\Bigg((N_1 + [\theta_2]_+N_2) \gamma &-\sqrt{2N_2}N_1\norm{\vmu_c}^2 \nonumber \\
&-\sqrt{18N}\cdot \frac{\sqrt{N} + t}{\sqrt{d}} - \sqrt{8 N_2}\left[-\theta_2\right]_+\Bigg),
\end{align}
where $ \left[ z \right]_+ = \max\{x,0\}$ denotes the positive part of $x$.
\end{lemma}
\begin{proof}[Proof of \lemref{lem:spurious_inner_prod_bnd}]
Our strategy for bounding $(E_1 + \theta_2E_2)^\top \mBeta$ begins with writing down the smallest value it can reach for any unit-norm classifier $\hat{\w}$ with normalized margin at least $\gamma$. Recalling that $\hat{\w} = \rmZ^\top \mBeta + \wperp$ (for $\wperp$ such that $\rmZ \wperp = 0$), the smallest possible value of $E_1^\top \mBeta$ is the solution to the following optimization problem:
\begin{align}\label{eq:betas_env1_bnd}
    \min_{\mBeta \in \R^N, \wperp\in \ker(\rmZ)}&~(E_1 + \theta_2E_2)^\top \mBeta \\
    \text{subject to }& \langle {\rmZ^\top \mBeta + \wperp}, y_i\x_i \rangle \geq \gamma \nonumber ~~\forall i\in [N] \\
    & \| \rmZ^\top \mBeta + \wperp \| = 1 \nonumber.
\end{align}
Since $\z_i = y_i \x_i$ and $\rmZ \w_\perp = 0$, the first constraint is equivalent to the vector inequality $Z Z^\top \mBeta \ge \gamma \ones$, and the second constraint is equivalent to $\mBeta^\top \rmZ \rmZ^\top \mBeta = 1-\norm{\wperp}^2$. Relaxing the second constraint, the smallest value of $\left(E_1 + \theta_2 E_2\right)^\top \mBeta$ is bounded from below by the solution to:

\begin{align}\label{eq:betas_env1_bnd_relax}
    \min_{\beta\in{\reals^N}} &~\mBeta^\top (E_1 + \theta_2E_2) \\
    \text{subject to }& \rmZ\rmZ^\top\mBeta \geq \gamma \mathbf{1} \nonumber \\
    & \mBeta^\top \rmZ\rmZ^\top\ \mBeta \leq 1. \nonumber
\end{align}
We now treat separately the two cases where $\theta_2 \leq 0$ and $\theta_2 > 0$.
\paragraph{The case where $\theta_2 \leq 0$.} Take Lagrange multipliers $\lambda\in{\reals_+^N}$ and $\nu \geq 0$, from strong duality the above equals:
\begin{align*}
    \max_{\lambda\in{\reals_+^N, \nu \geq 0}}\min_{\beta\in{\reals^N}}{\mBeta^\top (E_1 + \theta_2E_2) + \lambda^\top(\ones\gamma - \rmZ\rmZ^\top\mBeta) + \half\nu}(\mBeta^\top \rmZ \rmZ^\top \mBeta - 1)
\end{align*}
Optimizing the quadratic form over $\mBeta$, the above becomes:
\begin{align*}
    \max_{\lambda\in{\reals_+^N, \nu\geq 0}}{\lambda^\top\ones\gamma - \half\nu - \half\left( E_1 + \theta_2E_2-\rmZ\rmZ^\top\lambda\right)^\top\left(\nu \rmZ\rmZ^\top\right)^{-1}\left( E_1 + \theta_2E_2-\rmZ\rmZ^\top\lambda\right)}
\end{align*}
Maximizing over $\nu$ this becomes:
\begin{align} \label{eq:lem4_intermediate_lag}
    \max_{\lambda\in{\reals_+^N}}{\lambda^\top\ones\gamma -\sqrt{\left( E_1 + \theta_2E_2-\rmZ\rmZ^\top\lambda\right)^\top\left( \rmZ\rmZ^\top\right)^{-1}\left( E_1 + \theta_2E_2-\rmZ\rmZ^\top\lambda\right)}} := \max_{\lambda\in{\reals_+^N}}{\cL(\lambda)}
\end{align}

Thus, $(E_1 + \theta_2E_2)^\top \mBeta$ is lower bounded by $\cL(\lambda)$, for any $\lambda \in \reals_+^N$. Taking $\lambda = \alpha E_1$ for $\alpha=\left( 1 + \left( \norm{\vmu_c}^2 + \norm{\vmu_s}^2 \right)N_1 \right)^{-1}$, we obtain:
\begin{align*}
    \cL(\lambda) &= N_1 \gamma\alpha - \sqrt{\left( E_1 + \theta_2E_2-\alpha\rmZ\rmZ^\top E_1\right)^\top\left( \rmZ\rmZ^\top\right)^{-1}\left( E_1 + \theta_2E_2 - \alpha\rmZ\rmZ^\top E_1\right)} \\
    &\geq N_1\gamma\alpha - \sqrt{2} \norm{E_1 + \theta_2E_2 - \alpha\rmZ\rmZ^\top E_1} \\
    & \geq N_1\gamma\alpha - \sqrt{2} \norm{E_1 - \alpha\rmZ\rmZ^\top E_1} -\sqrt{2N_2}|\theta_2| \\
    &= N_1\gamma\alpha - \sqrt{2} \norm{\left(\I_N - \alpha \left(\E\left[\rmZ\rmZ^\top\right] + \rmZ\rmZ^\top-\E\left[\rmZ\rmZ^\top\right]\right)\right)E_1} - \sqrt{2N_2}|\theta_2|\\
    &\geq N_1\gamma\alpha -\sqrt{2}\norm{\left(\I_N -\alpha\E\left[\rmZ\rmZ^\top\right]\right)E_1} - \sqrt{2}\norm{\alpha\left( \rmZ\rmZ^\top - \E\left[ \rmZ\rmZ^\top\right]\right)E_1} -\sqrt{2N_2}|\theta_2|
\end{align*}
Here, the first inequality is from our assumption that \eqref{eq:assum_spectral} holds and hence $\rmZ\rmZ^\top \succeq \half\I_N$ and the second is a triangle inequality. Recall the bound $\norm{\rmZ\rmZ^\top - \E\left[\rmZ\rmZ^\top\right]}_{\mathrm{op}}\le 3 \frac{\sqrt{N}+t}{\sqrt{d}}$ from \lemref{lem:cov_spectral_bound} and apply it to obtain:
\begin{align*}
    \cL(\lambda) \geq& N_1\gamma\alpha -\sqrt{2}\norm{\left(\I_N -\alpha\E\left[\rmZ\rmZ^\top\right]\right)E_1} - \alpha -\alpha\sqrt{18N_1}\cdot \frac{\sqrt{N} + t}{\sqrt{d}} -\sqrt{2N_2}|\theta_2| \\
    \geq& N_1\gamma\alpha -\sqrt{2}\norm{\left(\I_N -\alpha\E\left[\rmZ\rmZ^\top\right]\right)E_1} - \alpha -\alpha\sqrt{18N}\cdot \frac{\sqrt{N} + t}{\sqrt{d}} -\sqrt{2N_2}|\theta_2|.
\end{align*}
Let us break down the second term in the bound above:
\begin{align*}
    \norm{\left(\I_N - \alpha\E\left[\rmZ\rmZ^\top\right]\right)E_1} &= \norm{\left(1-\alpha-\alpha N_1\norm{\vmu_s}^2\right)E_1 - \alpha N_1\norm{\vmu_c}^2\ones} \\
    &= \norm{\left(1-\alpha-\alpha N_1\norm{\vmu_s}^2\right)E_1 - \alpha N_1\norm{\vmu_c}^2\left(E_1 + E_2\right)} \\ 
    &= \sqrt{\left(1-\alpha\left(1+ N_1(\norm{\vmu_s}^2+\norm{\vmu_c}^2\right)\right)^2N_1 + \alpha^2 N_1^2\norm{\vmu_c}^4N_2} \\
    &=\alpha N_1\norm{\vmu_c}^2\sqrt{N_2},
\end{align*}
where the final equality used $\alpha\left(1+ N_1(\norm{\vmu_s}^2+\norm{\vmu_c}^2\right) = 1$. 
Overall, we get:
\begin{align} \label{eq:lem4_neg_theta_bnd}
    \mBeta^\top (E_1 + \theta_2 E_2) \geq \cL(\lambda) \geq \alpha \left(N_1\gamma -\sqrt{2N_2}N_1\norm{\vmu_c}^2 -\sqrt{18N}\cdot \frac{\sqrt{N} + t}{\sqrt{d}}\right) - \sqrt{2N_2}|\theta_2|.
\end{align}
The proof is complete by noting that $\alpha \ge 1/2$ due to \eqref{eq:assum_spectral}.
\paragraph{The case where $\theta_2 > 0$.} Before introducing Lagrange multipliers, we revisit \Cref{eq:betas_env1_bnd_relax} and this time we first bound the optimum from below as
\begin{align*}
\min_{\mBeta\in{C(\rmZ, \gamma)}}{\mBeta^\top(E_1 + \theta_2E_2)} \geq \theta_2\min_{\mBeta\in{C(\rmZ, \gamma)}}{ \mBeta^\top(E_1 + E_2)} + (1-\theta_2)\cdot\min_{\mBeta\in{C(\rmZ, \gamma)}}{\mBeta^\top E_1},
\end{align*}
where we used $C(\rmZ, \gamma)$ as a shorthand for the constraints in \Cref{eq:betas_env1_bnd_relax}. By our derivation for the case of $\theta_2 \leq 0$, the second term on the right hand side is readily bounded by the term in \Cref{eq:lem4_neg_theta_bnd} with $\theta_2 = 0$. We are left with bounding the first term, which turns out to be simpler since $E_1 + E_2 = \ones$. We repeat the process of taking the Lagrangian up until \Cref{eq:lem4_intermediate_lag}, and now choose $\lambda = \alpha \ones$ for $\alpha=\left(1 + \left(\norm{\vmu_c}^2 + \norm{\vmu_s}^2\right) N\right)^{-1}$. For completeness, let us rewrite the lower bound on the Lagrangian with these slight changes:
\begin{align*}
\cL(\lambda) &= N \gamma\alpha - \sqrt{\left( \ones -\alpha\rmZ\rmZ^\top \ones\right)^\top\left( \rmZ\rmZ^\top\right)^{-1}\left( \ones -\alpha\rmZ\rmZ^\top \ones \right)} \\
    &\geq N\gamma\alpha - \sqrt{2} \norm{\ones - \alpha\rmZ\rmZ^\top \ones} \\
    &= N\gamma\alpha - \sqrt{2} \norm{\left(\I_N - \alpha \left(\E\left[\rmZ\rmZ^\top\right] + \rmZ\rmZ^\top-\E\left[\rmZ\rmZ^\top\right]\right)\right)\ones} \\
    &\geq N\gamma\alpha -\sqrt{2}\norm{\left(\I_N -\alpha\E\left[\rmZ\rmZ^\top\right]\right)\ones} - \sqrt{2}\norm{\alpha\left( \rmZ\rmZ^\top - \E\left[ \rmZ\rmZ^\top\right]\right)\ones}
\end{align*}
Using again the bound on $\norm{\rmZ\rmZ^\top - \E\left[ \rmZ\rmZ^\top\right]}_{\mathrm{op}}$, the above bound becomes
\begin{align*}
    \cL(\lambda) \geq N\gamma\alpha -\sqrt{2}\norm{\left(\I_N -\alpha\E\left[\rmZ\rmZ^\top\right]\right)\ones} - \alpha -\alpha\sqrt{18 N}\cdot \frac{\sqrt{N} + t}{\sqrt{d}}.
\end{align*}
This time the chosen value for $\alpha$ makes the second term vanish, since
\begin{align*}
    \norm{\left(\I_N - \alpha\E\left[\rmZ\rmZ^\top\right]\right)\ones} &= \norm{\left( 1-\alpha-\alpha N\left(\norm{\vmu_s}^2 + \norm{\vmu_c}^2\right)\right) \ones} = 0.
\end{align*}
\Cref{eq:assum_spectral} again tells us that $\alpha > 1/2$ which leads us to the bound:
\begin{align*}
\mBeta^\top (E_1 + \theta_2 E_2) \geq \half \Bigg( \gamma(\theta_2 N + (1-\theta_2)N_1)  &- (1-\theta_2)\sqrt{2N_2}N_1\norm{\mu_c}^2 \\
 & - \sqrt{18}\left(\theta_2\sqrt{N} + (1-\theta_2)\sqrt{N_1}\right)\cdot \frac{\sqrt{N} + t}{\sqrt{d}} \Bigg) \\
\geq \half \Bigg( \gamma(\theta_2 N_2 + N_1) - \sqrt{2N_2}&N_1\norm{\mu_c}^2 - \sqrt{18N} \cdot \frac{\sqrt{N} + t}{\sqrt{d}} \Bigg).
\end{align*}
Combining the two cases for negative and positive $\theta_2$, we arrive at the desired bound in \Cref{eq:key-negative-bound}.
%
\end{proof}


\subsection{Controlling residual terms}




We now provide a bound on the terms in \eqref{eq:core_spu_ratio_expans} associated with quantities that vanish as the problem dimension grows.
\begin{lemma} \label{lem:residual_upper_bnd}
Conditioned on all the events in \corref{coro:ortho_bound} with parameter $t>0$ (which happen with probability at least $1-14\exp(-t^2 / 8)$) and the additional condition of \lemref{lem:cov_spectral_bound}, we have for $\vmu\in{\{\vmu_c, \vmu_s\}}$:
\begin{align} \label{eq:noise_bounds}
    \frac{1}{\norm{\vmu}^2}\left|\sum_{i\in{\idxset}}{\beta_i \langle n_i, \vmu \rangle} + \langle \wperp, \vmu \rangle\right| \leq \frac{3t}{\norm{\vmu}}\sqrt{\frac{{N}}{{d-N}}}
\end{align}
\end{lemma}
\begin{proof}
We prove the claim for $\vmu_s$; the proof for $\vmu_c$ is analogous.
Recall the random matrix $\G = \rmZ - \ones\vmu_c^\top - E_1\vmu_s^\top\in{\reals^{N\times d}}$ from \Cref{lem:high_prob_events}. From \eqref{eq:G_mus_bound} we get that $\norm{\G\vmu_s} \leq t\sqrt{\frac{N}{d}}\norm{\vmu_s}$ and then:\yc{$\beta \to \boldsymbol{\beta}$?}
\begin{align*}
    \sum_{i\in{\idxset}}{\beta_i\langle n_i, \vmu_s \rangle} = \mBeta^\top \G \vmu_s \leq \|\mBeta\|\|\G\vmu_s\| \leq t\norm{\mBeta}\sqrt{\frac{N}{d}}\norm{\vmu_s}.
\end{align*}
To eliminate $\norm{\mBeta}$ from this bound, we use $\rmZ \rmZ^\top \preceq \half I_N$ \yc{$I \to \rmI$?} due to \Cref{lem:cov_spectral_bound} to write
\begin{align*}
    \frac{1}{\sqrt{2}}\norm{\mBeta} 
    \le \sqrt{\mBeta^\top\rmZ\rmZ^\top\mBeta} 
    \le \sqrt{\mBeta^\top\rmZ^\top\rmZ\mBeta + \norm{\wperp}^2} = \norm{\hat{\w}} = 1.
\end{align*}

Finally, we use \eqref{eq:data_span_bound} from \Cref{coro:ortho_bound} to bound $\left| \langle \wperp, \vmu \rangle\right|$.
\end{proof}


\subsection{Proof of \texorpdfstring{\Cref{prop:neg_result_details}}{}} \label{sec:prop1_full_proof}

\yc{todo: restate}
\begin{repproposition}{prop:neg_result_details}
There are universal constants $c_r\in(0,1)$ and $C_d, C_r\in(1,\infty)$, such that, for any target normalized $\gamma$, $\theta_2$ such that $\theta_2 > - N_1\gamma / \sqrt{288 N_2}$, and failure probability $\delta\in(0,1)$, if 
\begin{align*} 
    &\max\{r_s^2, r_c^2\} \le \frac{c_r}{N}
    ~~,~~
    \frac{r_s^2}{r_c^2} \ge C_r \left( 1 + \frac{\sqrt{N_2}}{(N_1 + [\theta_2]_+ N_2)\gamma}\right)
    ~~\mbox{ and}~~ \\
    &d \geq C_d \frac{N}{\gamma^2 N_1^2 r_c^2}\log \frac{1}{\delta},
\end{align*}
then with probability at least $1-\delta$ over the drawing of $\vmu_c, \vmu_s$ and $(S_1, S_2)$ as described in \thmref{thm:main_statement}, any $\hat{\w}\in{\reals^d}$ that is a measurable function of $(S_1, S_2)$ and separates the data \yc{we should be clearer about what ``any model'' and ``the data'' means. Theorem 1 does a reasonable job I think but we should either repeat this here or refer to the appropriate formal definition} with normalized margin larger than $\gamma$ has robust error at least $0.5$.
\end{repproposition}
\begin{proof}[Proof of \Cref{prop:neg_result_details}]
Let $t\ge \sqrt{8\log \frac{14}{\delta}}$, so that the events described in the previous lemmas and corollaries all hold with probability at least $1-\delta$. Note that for $c_r \le 1/64$ we have 
\begin{equation}\label{eq:half-spectral-ass-1}
    \sqrt{N}(\norm{\vmu_c}+\norm{\vmu_s}) \le \frac{1}{4}
\end{equation}
and (since $\gamma \le \frac{1}{4\sqrt{N}}$)
\begin{equation*}
    d \ge \frac{C_d}{10} \frac{1}{\gamma^2} \frac{N t^2}{N_1^2\norm{\vmu_c}^2} 
    \ge \frac{C_d}{10c_r} \frac{N t^2}{N_1\gamma^2}  \ge
    \frac{16C_d}{10c_r} \frac{N^2 t^2}{N_1} N  \ge \frac 64C_d N t^2.
\end{equation*}
Consequently, for $C_d \ge 1$ 
\begin{equation}\label{eq:half-spectral-ass-2}
    \frac{\sqrt{N}+t}{\sqrt{d}} \le 2\sqrt{\frac{1}{64 C_d}} \le \frac{1}{4}.
\end{equation}
Combining \Cref{eq:half-spectral-ass-1,eq:half-spectral-ass-2}, we see that the condition in \eqref{eq:assum_spectral} holds. 

Therefore, we may apply \Cref{lem:spurious_inner_prod_bnd}; we now argue that under the assumptions of \Cref{prop:neg_result_details} the lower bound on $(E_1 + \theta_2E_2)^\top \mBeta$ simplifies to a constant multiple of $N_1 \gamma$. First, taking $c_r \le 1/9$ and $C_r \ge 1$, we have
\begin{equation*}
    \sqrt{2N_2} N_1\norm{\vmu_c}^2 \le 
    \frac{\sqrt{2N_2}N_1 \norm{\vmu_s}^2}{C_r \left( 1 + \frac{\sqrt{N_2}}{N_1\gamma}\right)}
    \le N_1 \gamma \frac{\sqrt{2}N_1 \norm{\vmu_s}^2}{C_r} \le  N_1 \gamma \frac{\sqrt{2}c_r}{C_r} \le
    \frac{1}{6} N_1 \gamma.
\end{equation*}
Second, using again $c_r\le 1/64$ and taking $C_d \ge 180$, 
\begin{equation*}
    \sqrt{18}\left([\theta_2]_+ \sqrt{N} + (1-[\theta_2]_+)\sqrt{N_1}\right) \frac{\sqrt{N}+t}{\sqrt{d}} \le 
    N_1 \gamma \frac{\sqrt{18}}{\sqrt{C_d/10}} \frac{\sqrt{N}+t}{t\sqrt{N}} \sqrt{N}\norm{\vmu_c} \le \frac{1}{6} N_1 \gamma.
    %
\end{equation*}
Finally, due to our condition on $\theta_2$ in the proposition statement, we have $\sqrt{8N_2}\left[ -\theta_2 \right]_+\leq N_1\gamma / 6$. Substituting all these into \eqref{eq:key-negative-bound}, we conclude that under our assumptions $(E_1 +\theta_2 E_2)^\top \mBeta \ge \frac{1}{4}(N_1 + [\theta_2]_+ N_2)\gamma$. 

Next, we combine the lower bound on $(E_1 + \theta_2 E_2)^\top \mBeta$ with \Cref{lem:residual_upper_bnd} to handle the denominator and numerator in the RHS of \eqref{eq:core_spu_ratio_expans}. Beginning with the numerator, we have
\begin{equation*}
    \ones^\top\mBeta + \frac{1}{\normcore^2}\left[\sum_{i\in{\idxset}}{\beta_i\langle n_i, \vmu_c \rangle} + \langle \wperp, \vmu_c \rangle\right] \le (E_1 + \theta_2 E_2)^\top \mBeta + (1-\theta_2)\norm{E_2}\norm{\mBeta} +  \frac{3t}{\normcore}\sqrt{\frac{N}{d-N}}. \\
\end{equation*}
As argued in the proof of \Cref{lem:residual_upper_bnd}, we have $\norm{\mBeta}\le \sqrt{2}$ and therefore $(1-\theta_2)\norm{E_2}\norm{\mBeta}\le \sqrt{8N_2}$. 
Substituting again our assumptions $d$ (which imply $d>2N$), and taking $C_d \ge 64\cdot 180$, we have
\begin{equation*}
    \frac{3t}{\normcore}\sqrt{\frac{N}{d-N}}\le \frac{\sqrt{18}t}{\norm{\vmu_c}}\sqrt{d}
    \le N_1 \gamma \sqrt{\frac{180}{C_d}} \le \frac{1}{8}N_1 \gamma.
\end{equation*}

For the denominator, noting $\norm{\vmu_c} \le \norm{\vmu_s}$ by our assumption, we may similarly write
\begin{equation*}
    (E_1 + \theta_2 E_2)^\top\mBeta + \frac{1}{\norm{\vmu_s}^2}\left[\sum_{i\in{\idxset}}{\beta_i\langle n_i, \vmu_s \rangle} + \langle \wperp, \vmu_s \rangle\right] \ge (E_1 + \theta_2 E_2)^\top\mBeta - \frac{1}{8}N_1 \gamma.
\end{equation*}
Consequently (since $(E_1 + \theta_2 E_2)^\top\mBeta \ge \frac{1}{4}N_1 \gamma$), we have that the denominator is nonnegative. (If the numerator is not positive, $\w$ will have error greater than $1/2$ for $\theta=0$). Substituting back to \eqref{eq:core_spu_ratio_expans} and using the lower bound $(E_1 + \theta_2 E_2)^\top\mBeta \ge \frac{1}{4}N_1 \gamma$, we get
\begin{align*}
\frac{\langle \w, \vmu_c\rangle}{\langle \w, \vmu_s\rangle} \frac{\norm{\vmu_s}^2}{\norm{\vmu_c}^2} \le \frac{(E_1 + \theta_2 E_2)^\top\mBeta + \sqrt{8N_2} + \frac{1}{8}N_1 \gamma}{(E_1 + \theta_2 E_2)^\top\mBeta - \frac{1}{8}N_1 \gamma} &\le 
\frac{ \frac{1}{4}\left(N_1 + [\theta_2]_+N_2\right)\gamma + \sqrt{8N_2} + \frac{1}{8}N_1 \gamma}{ \frac{1}{4}\left( N_1 + [\theta_2]_+N_2\right) \gamma - \frac{1}{8}N_1 \gamma} \\
& \le 3 + \frac{\sqrt{512 N_2}}{\left( N_1 + [\theta_2]_+N_2\right) \gamma}.
\end{align*}
Therefore, for $C_r \ge 32$ we have $\frac{\langle \w, \vmu_s\rangle}{\langle \w, \vmu_c\rangle} \ge 1$ as required. Since the error of classifier $\w$ in environment with parameter $\theta$ is
\begin{align*}
    Q\left(\frac{\inp{\w}{\vmu_c}}{\sigma\norm{\w}}\left( 1 + \theta\frac{\inp{\w}{\vmu_s}}{\inp{\w}{\vmu_c}} \right)\right),
\end{align*}
(where $Q(t) := \sP(\N(0; 1) > t)$ is the Gaussian tail function), the fact that $\frac{\inp{\w}{\vmu_s}}{\inp{\w}{\vmu_c}}\ge 1$ implies that there exists $\theta \in [-1,1]$ for which the error is $Q(0)=0.5$, implying the stated bound on the robust error.
\end{proof}

\section{Lower Bounds On the Achievable Margin} \label{sec:achievable_margin}
We now argue that, in our model, a simple signed-sample-mean estimator interpolates the data with normalized margin scaling as $1/\sqrt{N}$. This fact establishes the first part of \Cref{thm:main_statement}. 

\newcommand{\wmean}{\w_{\mathrm{mean}}}
\begin{proposition}\label{prop:achievable_margin}
There exist universal constants $c_r',C_d'>0$ such that, in the \learningsetting with parameters $N_1, N_2, d>0$,  $\vmu_c, \vmu_s\in\R^d$, $\theta_1,\theta_2\in[-1,1]$ and $\sigma^2 = 1/d$, for any $\delta\in (0,1/2)$ if
\begin{equation*}
    \max\{\normcore, \normspu\} \le \frac{c_r'}{N}~~,~~\theta_1 N_1 \ge -\theta_2 N_2~~\mbox{and}~~
    d \ge C_d' N^2 \log \left(\frac{1}{\delta}\right)
\end{equation*}
then with probability at least $1-\delta$, the signed-sample-mean estimator $\wmean = \frac{1}{N}\sum_{i=1}^{N} y_i x_i$ obtains normalized margin of at least $\frac{1}{\sqrt{8N}}$.  
\end{proposition}
\begin{proof}
Using the notation defined in the beginning of \Cref{sec:app_helpers}, \yc{make the notaiton a separate subsection?} we note that $\wmean = \frac{1}{N} \rmZ^\top \ones$ and (for $\sigma^2 d = 1$) its normalized margin is
\begin{equation*}
    \min_{i\in [N]} \frac{y_i \langle \x_i, \wmean \rangle}{\norm{\wmean}} =  \min_{i\in [N]} \frac{[\rmZ \wmean]_i }{\norm{\wmean}} = \min_{i\in [N]}\frac{[\rmZ \rmZ^\top \ones]_i}{\norm{\rmZ^\top \ones}}.
\end{equation*}
Substituting the assumed bounds on $d$ and $\norm{\vmu_c}, \norm{\vmu_s}$ into \Cref{lem:cov_spectral_bound} (with $t=\sqrt{8\log\frac{1}{\delta}}\ge \sqrt{2\log\frac{6}{\delta}}$), it is easy to verify that for sufficiently small $c_r'$ and sufficiently large $C_d'$, the condition in \eqref{eq:assum_spectral} holds, and therefore
\begin{equation*}
    \norm{\rmZ \rmZ^\top - \E \rmZ \rmZ^\top}_{\mathrm{op}} \le  3\frac{\sqrt{N}+t}{\sqrt{d}} \le \frac{1}{\sqrt{4N}},
\end{equation*}
with the final inequality following by choosing $C_d'$ sufficiently large. \Cref{lem:cov_spectral_bound} then also implies that $\rmZ \rmZ^\top \preceq 2 I_N$.

Noting that $\E \rmZ \rmZ^\top = I_N + \norm{\mu_c}^2 \ones \ones^\top + \norm{\mu_s}^2 (\theta_1 E_1 + \theta_2 E_2)(\theta_1 E_1 + \theta_2 E_2)^\top$, we have that, for all $i\in[N]$,
\begin{equation*}
    [\rmZ \rmZ^\top \ones]_i \ge [\E \rmZ \rmZ^\top \ones]_i - 
    \norm{\rmZ \rmZ^\top - \E \rmZ \rmZ^\top}_{\mathrm{op}}\norm{\ones} \ge 1 - 
    \frac{1}{\sqrt{4N}}\norm{\ones} = \frac{1}{2}.
\end{equation*}
Moreover, $\rmZ \rmZ^\top \preceq 2 I_N$ implies that
\begin{equation*}
    \norm{\rmZ^\top \ones} = \sqrt{\ones^\top \rmZ \rmZ^\top \ones} \le 2\norm{\ones} = 2\sqrt{N}.
\end{equation*}
Combining the above two displays yields the claimed margin bound.
\end{proof}

\yc{Maybe add a short discussion about the dimension probably only really needing to scale as $N \log^2\frac{N}{\delta}$ rather than $N^2 \log\frac{1}{\delta}$; we are getting the cruder bound by being heavy-handed with the Wishart concentration bounds}

\subsection{Margin for invariant classifiers}\label{ssec:invaraint-margin-calculation}

We now lower bound the margin achieved by the invariant classifier $\vw=\vmu_c$. To that end, note that $y_i\inp{\vmu_c}{\x_i} / \norm{\vmu_c} \sim \mathcal{N}(\norm{\vmu_c}; \sigma^2)$ for all $i\in \idxset$. Therefore, taking $\sigma^2 = 1/d$, with probability at least $1-\delta$ we have, for all $i\in \idxset$
\begin{equation*}
    \frac{y_i\inp{\vmu_c}{\x_i}}{\norm{\vmu_c} \sqrt{\sigma^2d}} \ge \norm{\vmu_c} - \frac{1}{\sqrt{d}} Q^{-1}\left(\frac{\delta}{N}\right) \ge 
    \norm{\vmu_c} -  \sqrt{\frac{2\log \frac{N}{\delta}}{d}}.
\end{equation*}

Substituting the choices
\begin{equation*}
    \norm{\vmu_c}^2 = \Theta\left( \frac{1}{N\left(1+\frac{\sqrt{N_2}}{N_1\gamma}\right)}\right)
    ~~\mbox{and}~~
    d = \Omega\left(\frac{\log\frac{1}{\delta}}{N_{\min}^2 \norm{\vmu_s}^4}\right) \ge \Omega\left(\frac{\log\frac{N}{\delta}}{\norm{\vmu_c}^2}\right)
\end{equation*}
form the proof of \Cref{thm:main_statement} (see \Cref{sec:app_main_statement_proof}), we obtain that, with probability at least $1-\delta$, the normalized margin of $\mu_c$ is 
\begin{equation*}
    \Omega(\norm{\vmu_c}) = \Omega\left( \frac{1}{\sqrt{N\left(1+\frac{\sqrt{N_2}}{N_1\gamma}\right)}}\right)=
    \Omega\left(\min\left\{\frac{1}{\sqrt{N}}, \sqrt{\frac{\gamma}{N_2}}\right\}\right).
\end{equation*}

In addition, letting $n_i\sim\mathcal{N}(0;\sigma^2 \rmI_d)$, we may also consider the invariant classifier
\begin{equation*}
    \vw = \vmu_c + \frac{1}{N} \sum_{i=1}^N n_i = \frac{1}{N}\sum_{e\in\{1,2\}} \sum_{(\vx,y)\in S_i} (y_i\vx_i - \theta_e \vmu_s).
\end{equation*}
The proof of \Cref{prop:achievable_margin} (with $\theta_1=\theta_2=0$) shows that, under the assumptions of that proposition, the classifier $\vw$ defined above attains margin of $\frac{1}{4\sqrt{N}}$ with high probability.

We emphasize once more that, while the discussion above shows the existence of invariant classifiers with good margin, our main results proves that these classifiers may be unlearnable from the finite samples $S_1$ and $S_2$. To show this existence numerically too, we run our simulations from \Cref{sec:validation} and add a model that is fitted on the features where $\vmu_s$ is removed (i.e. the new datapoints are $\rvx_i - \theta_e\vmu_s$ for each point $\rvx_i\in{S_e}$). We observe in \Cref{fig:invariant_interpolator_simulation} that for a sufficiently large dimension the model is both interpolating and has high robust accuracy, demonstrating the existence discussed above. We emphasize again that this model cannot necessarily be learned by an algorithm that receives the original $(S_1, S_2)$ (before the removal of $\vmu_s$).

\begin{figure}
    \centering
    \includegraphics[width=1.\linewidth]{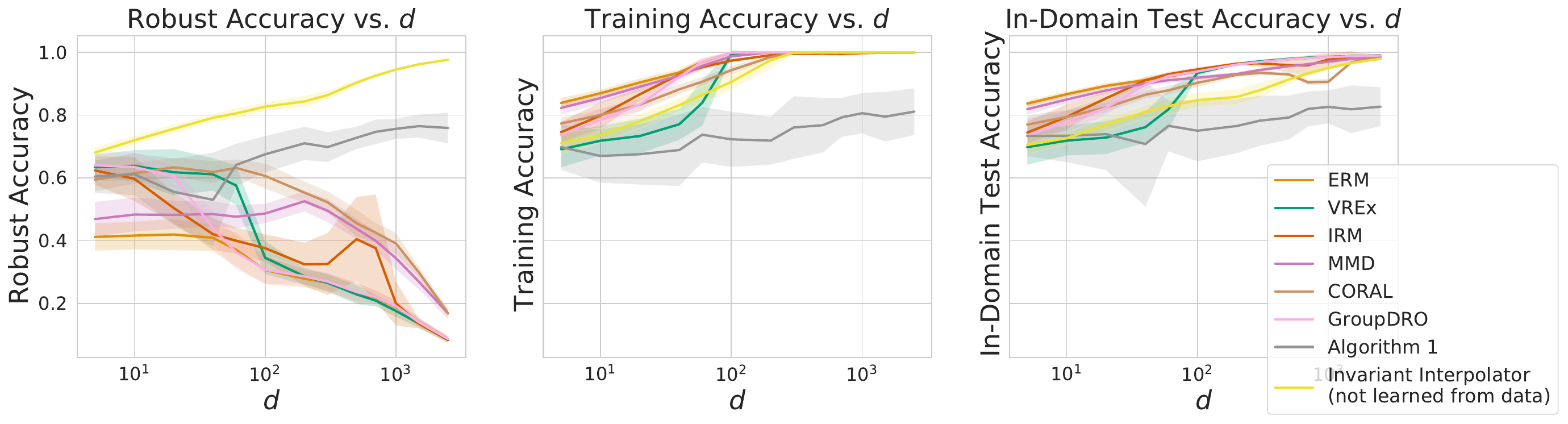}
    \caption{Simulation from \Cref{sec:validation} with an added model trained after removing the spurious feature. This demonstrated the existence of an invariant interpolator, yet our theoretical results suggest that this type of model cannot be learned by an interpolating learning rule}
    \label{fig:invariant_interpolator_simulation}
\end{figure}

\section{Analysis of \Cref{alg:two_phase_learning}}
The proof that Algorithm~\ref{alg:two_phase_learning} indeed achieves a non-trivial robust error will require some definitions and more mild assumptions which we now turn to describe.\\
\textbf{Definitions.} Denote the first-stage training set indices by $S$, where $|S|=N$ and second stage ``fine-tuning" set by $|D|=M$. Let us denote:
\begin{align*}
    \bar{\rvn}_e &= \frac{1}{N_{e}}\sum_{i\in{S_{e}}}{n_i}, ~ \bar{\rvm}_e = \frac{1}{M_{e}}\sum_{i\in{D_{e}}}{n_i}, ~ \bar{\rvm}_{e, 1} = \frac{1}{M_{e,1}}\sum_{i\in{D_{e,1}}}{n_i}.
\end{align*}
Models will be defined by:
\begin{align*}
    \w_e :=\:& \frac{1}{N_e}{\sum_{i\in{S_e}}{y_i\x_i}} = \mu_c + \theta_e\mu_s + \bar{\rvn}_e, \quad e\in{\{1,2\}}, \\
    f_{\rvv}(x; S) =\:&~ \langle v_1\cdot\w_1 + v_2\cdot\w_2, \x \rangle.
\end{align*}
The Equalized Opportunity (EOpp) constraint is:
\begin{align*}
    \hat{T}_1(f_\rvv; D, S) &= \hat{T}_2(f_\rvv; D) \\
    \hat{T}_e(f_\rvv; D, S) &= \frac{1}{M_{e,1}}\sum_{i\in{D_{e,1}}}{f_\rvv(\x_i)}
\end{align*}
\textbf{Additional Assumptions} We assume w.l.o.g $\theta_2 > \theta_1$, define $\Delta:=\theta_2-\theta_1 > 0$ and $r_{\mu} = \frac{\norm{\mu_s}}{\norm{\mu_c}} > 1$. We consider $r_\mu, \Delta$ as fixed numbers. That is, they do not depend on $N, d$ and other parameters of the problem. Also define $r:=\frac{\Delta\theta_{\max}}{\Delta + 4\theta_{\max}}$, where $\theta_{\max} := \mathrm{arg}{\max}\{|\theta_1|, |\theta_2|\} \leq 1$.
The following additional assumptions will be required for our concentration bounds.
\begin{assumption} \label{ass:norm_constraints_pos_result}
Let $t > 0$ be a fixed user specified value, which we define later and will control the success probability of the algorithm. We will assume that for each $e \in{\{1,2\}}$ and some universal constants $c_c, c_{s} > 0$:
\begin{align*}
    \norm{\mu_s}^2 &\geq t\sigma^2 c_s\max\left\{\frac{1}{r^2 N_e}, \frac{1}{(r\Delta)^2M_{e,1}}, \frac{\sqrt{d}}{M_{e, 1}r\Delta} \right\} \\
    \norm{\mu_c}^2 &\geq t\sigma^2 c_c\max\left\{\frac{1}{\Delta^2N_e}, \frac{r^2_\mu}{\left( \Delta^2 M_{e,1} \right)} , \frac{ r_{\mu}^2}{\Delta^2 M_e}, \frac{\sqrt{d}}{M_{e, 1}\Delta^2}, \frac{\sqrt{d}}{M_{e} \Delta} \right\}
\end{align*}
\end{assumption}

\textbf{Analyzing the EOpp constraint.} Writing the terms defined above in more detailed form gives:
\begin{align*}
    \epsilon_e(\rvv) =& \langle \bar{\rvm}_{e, 1}, v_1\left(\mu_c + \theta_1\mu_s + \bar{\rvn}_1\right) + v_2\left(\mu_c + \theta_2\mu_s + \bar{\rvn}_2\right) \rangle \\
    \delta_e(\rvv) =& \langle \bar{\rvm}_{e}, v_1\left(\mu_c + \theta_1\mu_s + \bar{\rvn}_1\right) + v_2\left(\mu_c + \theta_2\mu_s + \bar{\rvn}_2\right) \rangle \\
    \hat{T}_{e}(f_{\rvv}; D, S) =& (v_1 + v_2)\norm{\mu_c}^2 + (v_1\theta_1 + v_2\theta_2)\theta_e \norm{\mu_s}^2 + \\
    &\langle \mu_c + \theta_e\mu_s, v_1\bar{\rvn}_1 + v_2\bar{\rvn}_2 \rangle + \epsilon_e(\rvv)
\end{align*}
So the EOpp constraint is:
\begin{align} \label{eq:EOpp_constraint}
    v_1&\left[\theta_1\norm{\mu_s}^2 + \langle\bar{\rvn}_1, \mu_s\rangle\right]\theta_1 + v_2 \left[\theta_2\norm{\mu_s}^2 + \langle \bar{\rvn}_2, \mu_s \rangle\right] \theta_1 + \epsilon_1(\rvv) = \nonumber\\
    &v_1\left[\theta_1\norm{\mu_s}^2 + \langle\bar{\rvn}_1, \mu_s\rangle\right]\theta_2 + v_2 \left[\theta_2\norm{\mu_s}^2 + \langle \bar{\rvn}_2, \mu_s \rangle\right] \theta_2 + \epsilon_2(\rvv)
\end{align}

\begin{lemma} \label{lem:EOpp_unit_ball_intersection}
Consider all the solutions $\rvv = (v_1, v_2)$ that satisfy EOpp and have $\norm{\rvv}_\infty=1$. With probability $1$ there are exactly two such solutions $\rvv_{\mathrm{pos}}, \rvv_{\mathrm{neg}}$, where $\rvv_{\mathrm{pos}}=-\rvv_{\mathrm{neg}}$.
\end{lemma}
We will consider $\rvv_{\mathrm{pos}}$ as the solution that satisfies $v_{\mathrm{pos}, 1} + v_{\mathrm{pos}, 2} > 0$.
\begin{proof}
Is it easy to see that the EOpp constraint is a linear equation in $v_1, v_2$ and with probability $1$ the coefficients in this linear equations are nonzero. Therefore the solutions to this equation form a line in $\reals^2$ that passes through the origin. Consequently, this line intersects the $l_{\infty}$ unit ball at two points, that we denote $\rvv_{\mathrm{pos}}, \rvv_{\mathrm{neg}}$, which are negations of one another.
\end{proof}


\textbf{The proposed algorithm.} Now we can restate our algorithm in terms of $\rvv_{\mathrm{pos}}$ and $\rvv_{\mathrm{neg}}$ and analyze its retrieved solution.
\begin{itemize}
    \item Calculate $\w_1$ and $\w_2$ according to their definitions.
    \item Consider the solutions $\{\rvv_{\mathrm{pos}}, \rvv_{\mathrm{neg}}\}$ that satisfy EOpp and also $\norm{\rvv}_\infty=1$.
    \item Return the solution: $\rvv\in{\{\rvv_{\mathrm{pos}}, \rvv_{\mathrm{neg}}\}}$ which has the higher score, where the score is:
    \begin{align*}
       \rvv^* \in{\mathrm{arg}\max_{\rvv\in{\{\rvv_{\mathrm{pos}}, \rvv_{\mathrm{neg}}\}}}{\sum_{i\in{D}}{ \langle v_1\w_1 + v_2\w_2, y_i\x_i \rangle}}}
    \end{align*}
\end{itemize}

We first analyze the two possible solution $v_{\mathrm{pos}}$ and $v_{\mathrm{neg}}$ and show that their coordinates cannot be negations of each other. Intuitively, in an ideal scenario with infinite data, the EOpp constraint will enforce $v_1\theta_1 = -v_2\theta_2$. Then $v_1=-v_2$ is only possible if $\theta_1=\theta_2$, which we assume is not the case (if it is, we cannot identify the spurious correlation from data). The assumption of a fixed $\Delta>0$, will let us show that indeed with high probability $v_1=-v_2$ does not occur.
\begin{lemma} \label{lem:v_sum_bound}
Let $t>0$ and consider the solutions $v_{\mathrm{neg}},v_{\mathrm{pos}}$ that the algorithm may return. With probability at least \succesprobpos, the solutions satisfy $|v_1 + v_2| \geq \frac{\Delta}{2}$.
\end{lemma}
\begin{proof}
Assume that for $e\in{\{1, 2\}}$ the following events occur:
\begin{align} \label{eq:expected_loss_assum1}
|\langle\bar{\rvn}_e, \mu_s\rangle| &\leq r \norm{\mu_s}^2 \\
\label{eq:expected_loss_assum2}
|\langle \bar{\rvm}_{1,1} - \bar{\rvm}_{2,1}, \mu_c + \theta_e\mu_s + \bar{\rvn}_e \rangle| &\leq r\Delta \norm{\mu_s}^2
\end{align}

\corref{corro:positive_result_required_events} will show that they occur with the desired probability in our statement. Let us incorporate these events into the EOpp constraint. We group the items multiplied by $v_1$ and those multiplied by $v_2$:
\begin{align*}
    -v^*_1\big[\theta_1\norm{\mu_s}^2\Delta + \langle\bar{\rvn}_1, \mu_s\rangle&\Delta + \langle \bar{\rvm}_{1,1} - \bar{\rvm}_{2,1}, \mu_c + \theta_1\mu_s + \bar{\rvn}_1 \rangle \big] = \\
    & v^*_2 \left[\theta_2\norm{\mu_s}^2\Delta + \langle \bar{\rvn}_2, \mu_s \rangle \Delta + \langle \bar{\rvm}_{2,1} - \bar{\rvm}_{1,1}, \mu_c + \theta_2\mu_s + \bar{\rvn}_2 \rangle \right]
\end{align*}

Let us denote for convenience (where we drop the dependence on parameters in the notation):
\begin{align*}
    a = \norm{\mu_s}^{-2}\Delta \left( \langle\bar{\rvn}_1, \mu_s\rangle + \Delta^{-1}\langle \bar{\rvm}_{1,1} - \bar{\rvm}_{2,1}, \mu_c + \theta_1\mu_s + \bar{\rvn}_1 \rangle\right) \\
    b = \norm{\mu_s}^{-2}\Delta \left( \langle\bar{\rvn}_2, \mu_s\rangle + \Delta^{-1}\langle \bar{\rvm}_{2,1} - \bar{\rvm}_{1,1}, \mu_c + \theta_2\mu_s + \bar{\rvn}_2 \rangle\right)
\end{align*}
Now the EOpp constraint can be written as $-v_1^*\norm{\mu_s}^2\Delta\left( \theta_1 + a \right) = v_2^*\norm{\mu_s}^2\Delta\left( \theta_2 + b \right)$. Plugging in \eqref{eq:expected_loss_assum1} and \eqref{eq:expected_loss_assum2}, we see that $\max\{ |a|, |b| \} \leq r$.

Assume that $|\theta_1 + b| \geq |\theta_2 + a|$,
and note that since $\norm{\rvv^*}_\infty=1$ we have that $|v_1^*| = 1$ (the proof for the other case is analogous). \footnote{In the case where $|\theta_2 + a| \geq |\theta_1 + b|$ then $|v_2^*| = 1$ would hold.} We note that by definition $\Delta \leq 2\theta_{\max}$, hence if $v_2^*=0$ we have $|v_1^* + v_2^*| = 1 \geq \frac{\Delta}{2\theta_{\max}}$ and our claim holds. Otherwise, we can write:
\begin{align*}
    |v_1^* + v_2^*| &= \left|1 -\frac{\theta_2+b}{\theta_1+a}\right| = \left| \frac{\Delta + a - b}{\theta_1 + a} \right| \geq \frac{\Delta - 2r}{\theta_{\max} + r} = \frac{\Delta - 2\frac{\Delta\theta_{\max}}{\Delta + 4\theta_{\max}}}{\theta_{\max} + \frac{\Delta\theta_{\max}}{\Delta + 4\theta_{\max}}} \\
    &= \frac{\Delta\left(\Delta + 4\theta_{\max} - 2\theta_{\max}\right)}{\theta_{\max}\left(\Delta + 4\theta_{\max} + \Delta\right)} = \frac{\Delta}{2\theta_{\max}} \geq \frac{\Delta}{2}
\end{align*}
\end{proof}
The result above will be useful for proving the rest of our claims towards the performance guarantees of the algorithm. We first show that the retrieved solution is the one that is positively aligned with $\mu_c$.

\begin{lemma} \label{lem:retrieved_solution}
With probability at least \succesprobpos, between the two solutions considered at the second stage of our algorithm, the one with $v_1 + v_2 \geq 0$ achieves a higher score.
\end{lemma}
\begin{proof}
    Let's write down the score on environment $e\in{\{1, 2\}}$ in detail:
    \begin{align} \label{eq:emp_score_decompose}
       \sum_{i\in{D_e}}{\w^\top \x_i y_i} =& (v_1+v_2)\norm{\mu_c}^2 + \langle \mu_c , v_1\bar{\rvn}_1 + v_2\bar{\rvn}_2\rangle + \\ & (v_1\theta_1 + v_2\theta_2)\theta_e\norm{\mu_s}^2 + \langle \mu_s, \theta_e\left(v_1\bar{\rvn}_1 + v_2\bar{\rvn}_2\right) \rangle + \nonumber \\
       &\langle \bar{\rvm}_e, (v_1+v_2)\mu_c + (\theta_1v_1 + \theta_2v_2)\mu_s + v_1\bar{\rvn}_1 + v_2\bar{\rvn}_2 \rangle \nonumber
    \end{align}
We will bound all the items other than $(v_1 + v_2)\norm{\mu_s}^2$ with concentration inequalities, and for the second line also use the EOpp constraint. Regrouping items in \eqref{eq:EOpp_constraint} we have:
\begin{align*}
    \left| \left(v_1\theta_1 + v_2\theta_2\right) \norm{\mu_s}^2 + \langle \mu_s, v_1\bar{\rvn}_1 + v_2\bar{\rvn}_2 \rangle \right|\cdot \Delta = |\epsilon_2(\rvv) - \epsilon_1(\rvv)|
\end{align*}
In \corref{corro:positive_result_required_events} we will prove that with probability at least \succesprobpos, it holds that $|\epsilon_2(\rvv) - \epsilon_1(\rvv)| \leq \frac{\Delta}{6}|v_1 + v_2|\cdot\norm{\mu_c}^2$. Combined with $|\theta_e| < 1$, we get that the magnitude of the terms in the second line of \eqref{eq:emp_score_decompose} is bounded by $\frac{1}{6}|v_1 + v_2|\cdot\norm{\mu_c}^2$. We will also show in \corref{corro:positive_result_required_events} that the other two terms in \eqref{eq:emp_score_decompose} besides $(v_1 + v_2)\norm{\mu_c}^2$, are bounded by $\frac{1}{6}|v_1 + v_2|\cdot\norm{\mu_c}^2$. Hence we have for some $b$ such that $|b| \leq \half |(v_1 + v_2)|\cdot\norm{\mu_c}^2$ that:
\begin{align*}
    \sum_{i\in{D_e}}{\w^\top \x_i y_i} = (v_1 + v_2)\norm{\mu_c}^2 + b
\end{align*}
We note that the score in the algorithm is a weighted average of the scores over the training environments, yet the derivation above holds regardless of $e$. That is, $\theta_e$ did not play a role in the derivation other than the assumption that its magnitude is smaller than $1$. Hence it is clear that the solution $\rvv^*=\rvv_{\mathrm{pos}}$ will be chosen over $\rvv_{\mathrm{neg}}$.
\end{proof}

Once we have characterized our returned solution, it is left to show its guaranteed performance over all environments $\theta \in{[-1, 1]}$. We can draw a similar argument to \lemref{lem:retrieved_solution} to reason about the expected score obtained in each environment.
\begin{lemma}\label{lem:expected_score}
Let $t>0$ and consider the retrieved solution $\rvv^*$. With probability at least \succesprobpos,
the expected score of $\rvv^*$ over any environment corresponding to $\theta\in{[-1, 1]}$ is larger than $\frac{\Delta}{3}\norm{\mu_c}^2$.
\end{lemma}
\begin{proof}
The expected score can be written same as in \eqref{eq:emp_score_decompose}, except we can drop the last item since it has expected value $0$. We let $\theta\in{[-1, 1]}$ and write:
\begin{align*}
    \E_{\x,y\sim P_\theta}\left[ \w^\top \x y \right] = &(v^*_1 + v^*_2)\norm{\mu_c}^2 + \langle \mu_c , v^*_1\bar{\rvn}_1 + v^*_2\bar{\rvn}_2\rangle + \\
    & (v^*_1\theta_1 + v^*_2\theta_2)\theta\norm{\mu_s}^2 + \langle \mu_s, \theta\left(v^*_1\bar{\rvn}_1 + v^*_2\bar{\rvn}_2\right) \rangle \geq \frac{2}{3}(v^*_1 + v^*_2)\norm{\mu_c}^2.
\end{align*}
The inequality follows from the arguments already stated in \lemref{lem:retrieved_solution}, where the second and third items in the above expression have magnitude at most $\frac{1}{6}(v_1^* + v_2^*)\norm{\mu_c}^2$. Now it is left to conclude that $(v_1^* + v_2^*) \geq \frac{\Delta}{2}$, which is a direct consequence of \lemref{lem:v_sum_bound} and \lemref{lem:retrieved_solution}.
\end{proof}

\subsection{Proof of \Cref{prop:positive_result_main}} \label{sec:positive_proof_appendix}
Now we are in place to prove the guarantee given in the main paper on the robust error of the model returned by the algorithm. We will restate it here with compatible notation to the earlier parts of this section which slightly differ from those in the main paper (e.g. by incorporating $\Delta$). We also note that to obtain the statement in the main paper we should eliminate the dependence of Assumption ~\ref{ass:norm_constraints_pos_result} on $M_{e,1}$.
We do this by assuming that our algorithm draws $M_{e}$ as half of the original dataset for environment $e$. Then we have that $\sP(M_{e,1} \leq N_{\min} / 8)$ is bounded by  the cumulative probability of a Binomial variable with $k=N_{\min} / 8$ successes and at least $N_{\min}$ trials. This may be bounded with a Hoeffding bound by $1-2 \exp ( \half N_{\min} )$ and with a union bound over the two environments. To absorb this into our failure probability we require $N_{\min} > c_{eo}\log( 1/ \delta)$, leading to this added constraint in the main paper.
\begin{proposition} 
Under \assumref{ass:norm_constraints_pos_result}, let $\epsilon > 0$ be the target maximum error of the model and $t>0$. If $\norm{\mu_c}^2 \geq t Q^{-1}(\epsilon)\frac{15}{\Delta}\sigma^2\sqrt{\frac{d}{N_{\min}}}$, then with probability at least \succesprobpos the robust accuracy error of the model is at most $\epsilon$.
\end{proposition}
\begin{proof}
The error of the model in the environment defined by $\theta\in{[-1, 1]}$ is given by the Gaussian tail function:
\begin{align*}
    Q\left(\frac{\langle \w, \mu_c + \theta\mu_s \rangle}{\sigma\norm{\w}}\right)
\end{align*}
The nominator of this expression is simply the expected score from \lemref{lem:expected_score}, which we already proved is at least $\frac{\Delta}{3}\norm{\mu_c}^2$. Then we need to bound $\norm{\w}$ from above to get a bound on the robust accuracy. According to \corref{corro:positive_result_required_events}, if we denote $N_{\min} = \min\{N_1, N_2 \}$, this upper bound can be taken as $5t\sqrt{\sigma^2 d / N_{\min}}$. We plug this in to get:
\begin{align*}
    \frac{\langle \w, \mu_c + \theta \mu_s \rangle}{\sigma\norm{\w}} \geq \frac{\Delta}{15t}\norm{\mu_c}^2\frac{1}{\sigma^2}\sqrt{\frac{N_{\min}}{d}}
\end{align*}
Since $Q$ is a monotonically decreasing function, if $\norm{\mu_c}^2 \geq t Q^{-1}(\mathrm{\epsilon})\frac{15}{\Delta}\sigma^2\sqrt{\frac{d}{N_{\min}}}$ our model achieves the desired performance.
\end{proof}

\subsection{Required Concentration Bounds}
To conclude the proof we now show all the concentration results used in the above derivation.
Note that $\rvv^*$ is determined by all the other random factors in the problem, hence we should be careful when using them in our bounds. We will only use the fact that $\norm{\rvv^*}_\infty = 1$ and hence $\norm{\rvv^*}_1 \leq 2$.


To bound the inner product of noise vectors, we use \cite[Theorem 1.1]{rudelson2013hanson}:
\begin{theorem} (Hanson-Wright inequality). Let $X=\left(X_1, \ldots, X_n\right) \in \mathbb{R}^n$ be a random vector with independent components $X_i$ which satisfy $\mathbb{E} X_i=0$ and $\left\|X_i\right\|_{\psi_2} \leq$ $K$. Let $A$ be an $n \times n$ matrix. Then, for every $t \geq 0$,
$$
\mathbb{P}\left\{\left|X^{\top} A X-\mathbb{E}[ X^{\top} A X\right]|>t\right\} \leq 2 \exp \left[-c \min \left(\frac{t^2}{K^4\|A\|_{\mathrm{HS}}^2}, \frac{t}{K^2\|A\|}\right)\right]
$$
\end{theorem}
We can apply this theorem to get the following result.
\begin{corollary} \label{corr:noises_inner_prod}
for some universal constant $c>0$ (when we assume w.l.o.g that $M_{e'} \leq N_e$):
\begin{align} \label{eq:noise_inprod_bnd}
\sP \left\{ | \langle \bar{\rvn}_e, \bar{\rvm}_{e'} \rangle | > t \right\} \leq 2\exp\left[ -c \min\left(\frac{M_{e'}^2 t^2}{\sigma^4 d}, \frac{M_{e'} t}{\sigma^2\sqrt{d}}\right) \right]
\end{align}
\end{corollary}
\begin{proof}
We take $X$ as the concatenation of $\bar{\rvn}_e$ and $\bar{\rvm}_{e'}$, then $A$ is set such that $X^\top A X = \langle \bar{\rvn}_e, \bar{\rvm}_{e'} \rangle$ (e.g. $A_{i, i+d}=1$ for $1 \leq i \leq d$ and $0$ elsewhere).
Then $\norm{A}_{HS}^2 = d$ and $\norm{A}=\sqrt{d}$. Since entries in $\bar{\rvn}_e, \bar{\rvm}_{e'}$ are distributed as $\N(0, \frac{\sigma^2}{N_e}), \N(0, \frac{\sigma^2}{M_e})$ respectively, we have $K \leq C \frac{\sigma}{\sqrt{\min{\{N_e, M_{e'}\}}}}$ (assume w.l.o.g that $M_{e'} < N_e$) for some universal constant $C$ which we can incorporate into the constant $c$ in the theorem. This gives:
\begin{align*}
    \sP \left\{ | \langle \bar{\rvn}_e, \bar{\rvm}_{e'} \rangle | > t \right\} \leq 2\exp\left[ -c \min\left(\frac{M_{e'}^2 t^2}{\sigma^4 d}, \frac{M_{e'} t}{\sigma^2\sqrt{d}}\right) \right]
\end{align*}
\end{proof}

The next statement collects all of the concentration results we require for the other parts of the proof.
\begin{lemma} \label{lem:positive_result_concentrations}
Define $r:= \frac{\Delta \theta_{\max}}{\Delta + 4\theta_{\max}}$ where $\theta_{\max}:=\mathrm{arg}\max_{e\in{\{1,2\}}}\{|\theta_e|\}$, denote by $v^*$ the solution retrieved by the algorithm, and let $t>0$. When \assumref{ass:norm_constraints_pos_result} holds,
then with probability at least \succesprobpos we have that all the following events occur simultaneously (for all $e,e'\in{\{1, 2\}}$):
\begin{align}
    |\langle \bar{\rvn}_e, \mu_s \rangle| \leq& r \norm{\mu_s}^2 \label{eq:n_mu_assum}\\
    |\langle \bar{\rvn}_e, \mu_c \rangle| \leq& \frac{\Delta}{24} \norm{\mu_c}^2 \label{eq:n_muc_assum}\\
    |\langle \bar{\rvm}_{e,1}, \mu_c + \theta_{e'}\mu_s \rangle| \leq& \min \left\{ \frac{1}{4} r \Delta \norm{\mu_s}^2, \frac{\Delta}{36}\norm{\mu_c}^2 \right\} \label{eq:m_mus_assum}\\
    |\langle \bar{\rvm}_{e,1}, \mu_s  \rangle| \leq& \frac{\Delta}{64}\norm{\mu_c}^2 \label{eq:m_mu_spu_assum}\\
    |\langle \bar{\rvn}_e, \bar{\rvm}_{e',1} \rangle| \leq& \min\left\{\frac{1}{4} r\Delta \norm{\mu_s}^2, \frac{\Delta^2}{288} \norm{\mu_c}^2\right\} \label{eq:noise_prod_assum} \\
    \left| \langle \bar{\rvm}_e, \left(\mu_c + \theta_{e'}\mu_s\right) \rangle \right| \leq& \frac{1}{48}\Delta\cdot{\norm{\mu_c}}^2  \label{eq:m_mus_assum2}\\
    \left| \langle \bar{\rvn}_e, \bar{\rvm}_{e'} \rangle \right| \leq& \frac{1}{48}\Delta\cdot{\norm{\mu_c}}^2 \label{eq:noise_prod_assum2} \\
    \norm{\bar{\rvn}_e} \leq& t \sqrt{\frac{2\sigma^2 d}{N_e}} \label{eq:noise_norm_assum}
\end{align}
\end{lemma}
\begin{proof}

We first treat \eqref{eq:n_mu_assum} with a tail bound for Gaussian variables:
\begin{align*}
    \langle \bar{\rvn}_e, \mu_s
    \rangle \sim \N (0, \frac{\sigma^2\norm{\mu_s}^2}{N_e}) \Rightarrow \sP\left(|\langle \bar{\rvn}_e, \mu_s\rangle| > t_2 \right) \leq 2\exp\left( -\frac{t_2^2 N_e}{2\sigma^2\norm{\mu_s}^2} \right)
\end{align*}
Hence as long as $\norm{\mu_s}^2 \geq t\frac{2\sigma^2}{r^2 N_e}$, \eqref{eq:n_mu_assum} holds with probability at least $1-4\exp\{-t^2\}$ (since we take a union bound on the two environments). Following the same inequality and taking a union bound, \eqref{eq:n_muc_assum} also hold with probability at least $1-8\exp\{-t^2\}$ if $\norm{\mu_c}^2 \geq t\frac{1152\sigma^2}{\Delta^2N_e}$.

We use the same bound for \eqref{eq:m_mus_assum}, \eqref{eq:m_mu_spu_assum} and \eqref{eq:m_mus_assum2} while using $|\theta_e|\leq 1$. Hence for $t_2=\frac{1}{4} r \Delta \norm{\mu_s}^2$ and $t_2=\frac{\Delta}{36}\norm{\mu_c}^2$:
\begin{align*}
    \sP\left(|\langle \bar{\rvm}_{e,1}, \mu_c + \theta_{e'}\mu_s \rangle| > t_2 \right) &\leq 2\exp\left( -\frac{t_2^2 M_{e,1}}{2\sigma^2\norm{\mu_c + \theta_{e'}\mu_s}^2} \right) = 2\exp\left( -\frac{(r\Delta)^2\norm{\mu_s}^4 M_{e,1}}{32\sigma^2\norm{\mu_c + \theta_{e'}\mu_s}^2} \right) \\
    &\leq 2\exp\left( -\frac{(r\Delta)^2\norm{\mu_s}^2 M_{e,1}}{128\sigma^2} \right) \\
    \sP\left(|\langle \bar{\rvm}_{e,1}, \mu_c + \theta_{e'}\mu_s \rangle| > t_2 \right) &\leq 2\exp\left( -\frac{\Delta^2 \norm{\mu_c}^4 M_{e,1}}{2592\sigma^2\norm{\mu_c + \theta_{e'}\mu_s}^2} \right) = 2\exp\left( -\frac{\Delta^2 \norm{\mu_c}^2 M_{e,1}}{10368\sigma^2r^2_\mu} \right)
\end{align*}
Similarly with $t_2 = \frac{1}{48}\Delta\cdot\norm{\mu_c}^2$:
\begin{align*}
    \sP\left( \left| \langle \bar{\rvm}_e, \left(\mu_c + \theta_{e'}\mu_s\right) \rangle \right| > t_2 \right) \leq 2\exp\left( -\frac{\Delta^2\norm{\mu_c}^4 M_{e}}{(48\sigma\norm{\mu_c + \theta_{e'}\mu_s})^2} \right)
\end{align*}
Taking the required union bounds we get that with probability at least $1-24\exp{\left( -t^2 / 2 \right)}$ \eqref{eq:m_mus_assum}, \eqref{eq:m_mu_spu_assum} and \eqref{eq:m_mus_assum2} hold, as long as $\norm{\mu_s}^2 \geq t\cdot 128\sigma^2((r\Delta)^2 M_{e,1})^{-1}$ and $\norm{\mu_c}^2 \geq t\cdot\max\left\{10368\sigma^2r^2_\mu\left( \Delta^2 M_{e,1} \right)^{-1} , (96\sigma r_{\mu})^2(\Delta^2 M_e)^{-1}\right\}$.

For \eqref{eq:noise_prod_assum} and \eqref{eq:noise_prod_assum2} we use \corref{corr:noises_inner_prod}: \footnote{For simplicity, assume we have $\sqrt{M_{1,1}^{-2} + M_{2,1}^{-2}} \leq N^{-1}_1$ and that we set $t$ large enough such that $\left(M_{1,1}^{-1} + M_{2,1}^{-1}\right)^{-2}t^2 / (\sigma^4 d) \geq \left(M_{1,1}^{-1} + M_{2,1}^{-1}\right)^{-1}t / (\sigma^2 \sqrt{d})$}
\begin{align*}
    \sP\left\{ |\langle \bar{\rvn}_e, \bar{\rvm}_{e',1} \rangle| \geq t_2  \right\} \leq 2\exp \left[ -c \frac{M^2_{e', 1}t_2^2}{\sigma^4 d} \right]
\end{align*}
Setting $t_2=\frac{r\Delta}{4}\norm{\mu_s}^2$ or $t_2=\frac{\Delta^2}{288}\norm{\mu_c}^2$ we will get that:
\begin{align*}
    &\sP\left( |\langle \bar{\rvn}_e, \bar{\rvm}_{e',1} \rangle| \geq \min\left\{ \frac{r\Delta}{4}\norm{\mu_s}^2,  \frac{\Delta^2}{288}\norm{\mu_c}^2\right\} \right) \leq \\
    &\qquad\qquad 2\exp\left( -c\frac{M^2_{e',1}}{\sigma^4 d} \min \left\{ \frac{(r\Delta)^2}{16}\norm{\mu_s}^4, \frac{\Delta^4}{288^2}\norm{\mu_c}^4\right\} \right)
\end{align*}
Hence we require $\norm{\mu_c}^2 \geq t\cdot c \cdot(M_{e', 1}\Delta^2)^{-1}\cdot(288 \sigma^2\sqrt{d})$ and $\norm{\mu_s}^2 \geq t\cdot c \cdot(M_{e', 1}r\Delta)^{-1}\cdot(4 \sigma^2\sqrt{d})$ for \eqref{eq:noise_prod_assum} to hold. For \eqref{eq:noise_prod_assum2} we can get in a similar manner that it holds in case that $\norm{\mu_c}^2 \geq t \cdot c \cdot (M_{e'} \Delta)^{-1}(48\sigma^2 \sqrt{d})$. The probability for all the events listed so far to occur is at last $1-32\exp\left(-t^2/2\right)$. Finally, for \eqref{eq:noise_norm_assum} we simply use the bound on a norm of Gaussian vector:
\begin{align*}
    \sP\left( \norm{\bar{\rvn}_e} \geq t_2\right) \leq 2\exp\left(-\frac{t_2^2N_e}{2\sigma^2 d}\right)
\end{align*}
Plugging in $t\sqrt{\frac{2\sigma^2 d}{N_e}}$ we arrive at the desired result with a final union bound that give the overall probability of at least $1-34\exp\left(-t^2/2\right)$.
\end{proof}

We now use the bounds above to write down the specific bounds on expressions that we used during proof.
\begin{corollary} \label{corro:positive_result_required_events}
Conditioned on all the events in \lemref{lem:positive_result_concentrations}, we have for $e\in{\{1,2\}}$ that:
\begin{align}
\frac{\Delta}{6}|v_1 + v_2|\cdot\norm{\mu_c}^2 &\geq |\epsilon_2(\rvv) - \epsilon_1(\rvv)| \label{eq:event_corro_eq1}\\
\frac{1}{6}|v_1 + v_2|\cdot\norm{\mu_c}^2 &\geq |\langle \mu_c , v_1\bar{\rvn}_1 + v_2\bar{\rvn}_2\rangle| \label{eq:event_corro_eq2}\\
\frac{1}{6}|v_1 + v_2|\cdot\norm{\mu_c}^2 &\geq |\langle \bar{\rvm}_e, (v_1+v_2)\mu_c + (\theta_1v_1 + \theta_2v_2)\mu_s + v_1\bar{\rvn}_1 + v_2\bar{\rvn}_2 \rangle | \label{eq:event_corro_eq3}\\
r\Delta \norm{\mu_s}^2 &\geq |\langle \bar{\rvm}_{1,1} - \bar{\rvm}_{2,1}, \mu_c + \theta_e\mu_s + \bar{\rvn}_e \rangle| \label{eq:event_corro_eq4}\\
r \norm{\mu_s}^2 &\geq |\langle\bar{\rvn}_e, \mu_s\rangle| \label{eq:event_corro_eq5}\\
5t\sqrt{\frac{\sigma^2 d}{\min_e{N_e}}} &\geq \norm{\w} \label{eq:event_corro_eq6}
\end{align}
\end{corollary}

\begin{proof}
\eqref{eq:event_corro_eq5} is just \eqref{eq:n_mu_assum} restated for convenience. \eqref{eq:event_corro_eq4} is a combination of \eqref{eq:m_mus_assum} and \eqref{eq:noise_prod_assum}:
\begin{align*}
    |\langle \bar{\rvm}_{1,1} - \bar{\rvm}_{2,1}, \mu_c + \theta_e\mu_s + \bar{\rvn}_e \rangle| \leq \sum_{e'} | \langle \bar{\rvm}_{e',1}, \mu_c + \theta_e\mu_s \rangle | + | \langle \bar{\rvm}_{e',1}, \bar{\rvn}_e \rangle| \leq r\Delta\norm{\mu_s}^2
\end{align*}
These are the events required for \lemref{lem:v_sum_bound}, hence from now on we can now assume that:
\begin{align*}
    |v_1 + v_2| \geq \frac{\Delta}{2} = \frac{\Delta}{4}\cdot 2 \geq \frac{\Delta}{4}\norm{\rvv}_1
\end{align*}
Now we can combine with \eqref{eq:n_muc_assum} to prove \eqref{eq:event_corro_eq2}:
\begin{align*}
    \langle \mu_c, v_1\bar{\rvn}_1 + v_2\bar{\rvn}_2 \rangle \leq \sum_{e}{|v_e|\cdot|\langle \mu_c, \bar{\rvn}_e \rangle|} \leq \norm{\rvv}_1 \frac{\Delta}{24}\norm{\mu_c}^2 \leq \frac{1}{6}|v_1 + v_2|\cdot\norm{\mu_c}^2
\end{align*}
Next we prove \eqref{eq:event_corro_eq3} in a similar manner using \eqref{eq:m_mus_assum2} and \eqref{eq:noise_prod_assum2}:
\begin{align*}
    |\langle \bar{\rvm}_e,& (v_1+v_2)\mu_c + (\theta_1v_1 + \theta_2v_2)\mu_s + v_1\bar{\rvn}_1 + v_2\bar{\rvn}_2 \rangle| \leq \\
    &\sum_{e'}{|v_{e'}|\cdot\left(|\langle \bar{\rvm}_e, \mu_c + \theta_{e'}\mu_s \rangle| + |\langle \bar{\rvm}_e, \bar{\rvn}_{e'} \rangle|\right)} \leq \norm{\rvv}_1\cdot 2\cdot\frac{1}{48}\Delta\norm{\mu_c}^2 \leq \frac{1}{6}|v_1 + v_2|\cdot\norm{\mu_c}^2
\end{align*}

For \eqref{eq:event_corro_eq1}, let us write the right hand side:
\begin{align*}
\left| \epsilon_2(\rvv) - \epsilon_1(\rvv) \right| &= \left| \langle \bar{\rvm}_{2,1}-\bar{\rvm}_{1,1}, v_1(\mu_c + \theta_1\mu_s +\bar{\rvn}_1) + v_2(\mu_c + \theta_2\mu_s +\bar{\rvn}_2) \rangle \right| \\
& = |(v_1 + v_2)\cdot\langle \bar{\rvm}_{2,1}-\bar{\rvm}_{1,1}, \mu_c + \half(\theta_1+\theta_2)\mu_s \rangle \\
& + \langle \bar{\rvm}_{2,1}-\bar{\rvm}_{1,1}, v_1\bar{\rvn}_1 + v_2\bar{\rvn}_2 \rangle + \half(v_1-v_2)\langle \bar{\rvm}_{2,1}-\bar{\rvm}_{1,1}, \Delta\mu_s \rangle| \\
& \leq |v_1 + v_2|\cdot\sum_{e}{|\langle \bar{\rvm}_{e,1}, \mu_c + \half(\theta_1+\theta_2\mu_s) \rangle|} + \norm{\rvv}_1\sum_{e,e'}{|\langle \bar{\rvm}_{e,1}, \bar{\rvn}_{e'} \rangle|} \\
& + \half \Delta \norm{\rvv}_1 \sum_{e}{|\langle \bar{\rvm}_{e,1}, \mu_s  \rangle|} \\
& \leq |v_1 + v_2|\cdot\sum_{e}{|\langle \bar{\rvm}_{e,1}, \mu_c + \half(\theta_1+\theta_2\mu_s) \rangle|} + \frac{4}{\Delta}|v_1 + v_2|\sum_{e,e'}{|\langle \bar{\rvm}_{e,1}, \bar{\rvn}_{e'} \rangle|} \\
& + 2 |v_1 + v_2| \sum_{e}{|\langle \bar{\rvm}_{e,1}, \mu_s  \rangle|} \\
&\leq \frac{1}{6}\Delta|v_1 + v_2|
\end{align*}
The first inequality is simply a triangle inequality, the second plugs in the bound we obtained for $\norm{\rvv}_1$ and the last uses the relevant inequalities from \lemref{lem:positive_result_concentrations}.

For \eqref{eq:event_corro_eq6}, we write the weights of the returned linear classifier as:
\begin{align*}
    \w = v^*_1(\mu_c + \theta_1\mu_s + \bar{\rvn}_1) + v^*_2(\mu_c + \theta_2\mu_s + \bar{\rvn}_2)
\end{align*}
Hence we can bound:
\begin{align*}
\norm{\w} -& (v^*_1+v^*_2) \norm{\mu_c} \leq \norm{(v^*_1\theta_1 + v^*_2\theta_2)\mu_s + v^*_1\bar{\rvn}_1 + v^*_2\bar{\rvn}_2} \\
&= \sqrt{(v_1^*\theta_1 + v_2^*\theta_2)^2\norm{\mu_s}^2 + 2\langle v_1^*\bar{\rvn}_1 + v_2^*\bar{\rvn}_2 , (v_1^*\theta_1 + v_2^*\theta_2)\mu_s \rangle + \norm{v_1^*\bar{\rvn}_1 + v_2^*\bar{\rvn}_2}^2} \\
&=\sqrt{(v_1^*\theta_1 + v_2^*\theta_2)\left((v_1^*\theta_1 + v_2^*\theta_2)\norm{\mu_s}^2 + 2\langle v_1^*\bar{\rvn}_1 + v_2^*\bar{\rvn}_1, \mu_s \rangle\right) + \norm{v_1^*\bar{\rvn}_1 + v_2^*\bar{\rvn}_2}^2}
\end{align*}
We also proved in \lemref{lem:retrieved_solution}, that under the events we assumed and the EOpp constraint:
\begin{align*}
    (v_1^*\theta_1 + v_2^*\theta_2)\norm{\mu_s}^2 + 2\langle v_1^*\bar{\rvn}_1 + v_2^*\bar{\rvn}_2, \mu_s \rangle &\leq 2\left((v_1^*\theta_1 + v_2^*\theta_2)\norm{\mu_s}^2 + |\langle v_1^*\bar{\rvn}_1 + v_2^*\bar{\rvn}_2, \mu_s \rangle|)\right) \\
    &\leq \frac{1}{3}(v_1^* + v_2^*)\norm{\mu_c}^2
\end{align*}
Incorporating with $v_1^*\theta_1 + v_2^*\theta_2 \leq 2(v_1^* + v_2^*)$, the concavity of the square root and \eqref{eq:noise_norm_assum}, we get:
\begin{align*}
    \norm{\w} &\leq \left(1 + \sqrt{2 / 3}\right)(v_1^* + v_2^*)\norm{\mu_c} + \norm{v_1^*\bar{\rvn}_1 + v_2^*\bar{\rvn}_2} \\
    &\leq \left(1 + \sqrt{2 / 3}\right)(v_1^* + v_2^*)\norm{\mu_c} + \norm{\bar{\rvn}_1} + \norm{\bar{\rvn}_2} \\
    &\leq \left(1 + \sqrt{2 / 3}\right)(v_1^* + v_2^*)\norm{\mu_c} + t\cdot\sqrt{\frac{\sigma^2 d}{\min_e{N_e}}} \\
    &\leq 4\norm{\mu_c} + t\cdot\sqrt{\frac{\sigma^2 d}{\min_e{N_e}}} \\
    &\leq 5t\cdot\sqrt{\frac{\sigma^2 d}{\min_e{N_e}}}
\end{align*}
\end{proof}
\section{Proof of \Cref{thm:main_statement}}\label{sec:app_main_statement_proof}
\yc{todo: restate}
\begin{proof}[Proof of \Cref{thm:main_statement}]
Our proof simply consists of choosing the free parameters in \Cref{thm:main_statement} ($r_c,r_s,d$ and $\sigma$) based on \Cref{prop:neg_result_details,prop:positive_result_main,prop:achievable_margin} such that all the claims in the theorem hold simultaneously. Keeping in line with the setting of \Cref{prop:neg_result_details,prop:achievable_margin}, we take $\sigma^2 = 1/d$. Next, our strategy is to pick $r_s$ and $r_c$ so as to satisfy the requirements of \Cref{prop:neg_result_details,prop:achievable_margin}, and then pick a sufficiently large $d$ so that the requirements of \Cref{prop:positive_result_main} hold as well. Throughout, we set $\delta=99/100$ so as to meet the failure probability requirement stated in the theorem; it is straightforward to adjust the proof to guarantee lower error probabilities. 

Starting with the value of $r_s$, we let
\begin{equation*}
    r_s^2 = \frac{\min\{c_n, c_n'\}}{N}
\end{equation*}
where the parameters $c_n, c_m$ and $c_n'$ are as given by \Cref{prop:neg_result_details,prop:achievable_margin}, respectively.
Next, we pick $r_c$ to be
\begin{equation*}
    r_c^2 = 
    \frac{r_s^2}{C_r \left(1+ \frac{\sqrt{N_2}}{N_1\gamma} \right)}
    = \frac{\min\{c_n, c_n'\}}{C_r N \left(1+ \frac{\sqrt{N_2}}{N_1\gamma} \right)}
\end{equation*}
with $C_r$ from \Cref{prop:neg_result_details} (this setting guarantees $r_c \le r_s$ as $C_r\ge 1$). Thus, we have satisfied the requirements in \eqref{eq:negative_res_norm_constraint} in \Cref{prop:neg_result_details}, as well as the requirement $\max\{r_c,r_s\} \le \frac{c_n'}{N}$ in \Cref{prop:achievable_margin}; it remains to choose $d$ so that the remaining requirements hold.

\Cref{prop:neg_result_details} requires the dimension to satisfy $d \ge C_d \frac{N}{\gamma^2 N_1^2 r_c^2} \log\frac{1}{\delta}$ and \Cref{prop:achievable_margin} requires $d \ge C_d' N^2 \log\frac{1}{\delta}$. Substituting our choices of $\sigma^2 = 1/d$, $r_s$ and $r_c$ above, let us rewrite the requirements of \Cref{prop:positive_result_main} as lower bounds on $d$. The requirement in \eqref{eq:mus_constraint_positive} reads
\begin{equation*}
    d \ge C_s^2 \frac{\log\frac{1}{\delta}}{N_{\min}^2 r_s^4},
\end{equation*}
while the requirement in \cref{eq:muc_constraint_positive} (with minor simplifications) reads
\begin{equation*}
    d \ge \frac{C_c^2 \log\frac{1}{\delta}}{N_{\min}r_c^4 } \max\left\{ (Q^{-1}(\epsilon))^2, \frac{1}{N_{\min}}, r_s^2
    \right\}.
\end{equation*}
Using $r_s \ge r_c$ and $r_s^2 \le \frac{1}{N_{\min}}$, the above two displays simplify to
\begin{equation*}
    d \ge \frac{\max\{C_c, C_s\}^2 \log\frac{1}{\delta}}{N_{\min}r_c^4 } \max\left\{ (Q^{-1}(\epsilon))^2, \frac{1}{N_{\min}}\right\}.
\end{equation*}
Therefore, taking
\begin{equation*}
    d = \max\{C_d, C_d', C_s^2, C_c^2\} \max\left\{ N^2, \frac{N}{\gamma^2 N_1^2 r_c^2}, \frac{(Q^{-1}(\epsilon))^2}{N_{\min} r_c^4}, \frac{1}{N_{\min}^2 r_c^4}\right\}\log \frac{1}{\delta}
\end{equation*}
fulfills all the requirements and completes the proof.
\end{proof}
\section{Definitions of Invariance and Their Manifestation In Our Model}\label{sec:invariance_defs}
In \secref{sec:positive_result} we show that the Equalized Odds principle in our setting reduces to the demand that $\inp{\w}{\vmu_s}=0$. Here we provide short derivations that show this is also the case for some other invariance principles from the literature. We will show this in the population setting, that is in expectation over the training data. We also assume that $\theta_1 \neq \theta_2$.
\paragraph{Calibration over environments \citep{wald2021calibration}} Assume $\sigma(\inp{\w}{\x})$ is a probabilistic classifier with some invertible function $\sigma:\reals\rightarrow [0, 1]$ such as a sigmoid, that maps the output of the linear function to a probability that $y=1$. Calibration can be written as the condition that:
\begin{align*}
    \sP_{\theta}(y=1 \mid \sigma(\inp{\w}{\x} - b)=\hat{p}) = \hat{p} \quad \forall \hat{p}\in{[0,1]}.
\end{align*}
Calibration on training environments in our setting then requires that this holds simultaneously for $\sP_{\theta_1}$ and $\sP_{\theta_2}$. We can write the conditional probability of $y$ on the prediction (when the prior over $y$ is uniform) as:
\begin{align*}
    \sP_{\theta_e}(y=1 \mid \inp{\w}{\x} - b = \alpha) = \frac{\exp\left(\frac{\left(\alpha-\inp{\w}{\vmu_c + \theta_1 \vmu_s} + b\right)^2}{2\sigma^2\norm{\w}^2}\right)}{\exp\left(\frac{\left(\alpha-\inp{\w}{\vmu_c + \theta_1 \vmu_s} + b\right)^2}{2\sigma^2\norm{\w}^2}\right) + \exp\left(\frac{\left(\alpha+\inp{\w}{\vmu_c + \theta_1 \vmu_s} + b\right)^2}{2\sigma^2\norm{\w}^2}\right)}
\end{align*}

Now it is easy to see that if the classifier is calibrated across environments, we must have equality in the log-odds ratio for the above with $e=1$ and $e=2$ and all $\alpha\in{\reals}$:
\begin{align*}
    \frac{\left(\alpha-\inp{\w}{\vmu_c + \theta_1 \vmu_s} + b\right)^2}{2\sigma^2\norm{\w}^2} - \frac{\left(\alpha+\inp{\w}{\vmu_c + \theta_1 \vmu_s} + b\right)^2}{2\sigma^2\norm{\w}^2} = \\
    \frac{\left(\alpha-\inp{\w}{\vmu_c + \theta_2 \vmu_s} + b\right)^2}{2\sigma^2\norm{\w}^2} - \frac{\left(\alpha+\inp{\w}{\vmu_c + \theta_2 \vmu_s} + b\right)^2}{2\sigma^2\norm{\w}^2}.
\end{align*}
After dropping all the terms that cancel out in the subtractions we arrive at:
\begin{align*}
    \inp{\w}{\vmu_c + \theta_1 \vmu_s} = \inp{\w}{\vmu_c + \theta_2 \vmu_s}.
\end{align*}
Clearly this holds if and only if $\inp{\w}{\vmu_s} = 0$, hence calibration on both environments entails invariance in the context of the data generating process of \Cref{def:learning_setup}.

\paragraph{Conditional Feature Matching \citep{li2018deep, veitch2021counterfactual}} Treating the environment index as a random variable, the conditional independence relation $\inp{\w}{\x} \indep e \mid y$ is a popular invariance criterion in the literature. Other works besides the ones mentioned in the title of this paragraph have used this, like the Equalized Odds criterion \citep{hardt2016equality}. This independence is usually enforced w.r.t available training distributions, hence in our case w.r.t $\sP_{\theta_1}, \sP_{\theta_2}$. Writing this down we can see that:
\begin{align*}
    \sP_{\theta_e}(\inp{\w}{\x} \mid y=1) = \N (\inp{\w}{\mu_c + \theta_e\mu_s}, \norm{\w}^2\sigma^2I).
\end{align*}
Hence requiring conditional independence in the sense of $\sP_{\theta_1}(\inp{\w}{\x} \mid y=1)=\sP_{\theta_2}(\inp{\w}{\x} \mid y=1)$ means we need to have equality of the expectations, i.e. $\inp{\w}{\mu_c + \theta_1\mu_s} = \inp{\w}{\mu_c + \theta_2\mu_s}$ which happens only if $\inp{\w}{\mu_s}=0$.
\paragraph{Other notions of invariance.} It is easy to see that even without conditioning on $y$, the independence relation $\inp{\w}{\x} \indep e$ used in \citet{veitch2021counterfactual} among many others will also require that $\inp{\w}{\vmu_s}=0$. For the last invariance principle we discuss here, we note that VREx and CVaR Fairness essentially require equality in distribution of losses \citep{pmlr-v97-williamson19a, krueger2020out} under both environments. Examining the expression for the error of $\w$ under our setting (\eqref{eq:error_explicit}) reveals immediately that these conditions will also impose $\inp{\w}{\vmu_s}=0$.

\section{Invariant Risk Minimization and Maximum Margin} \label{sec:irm_max_margin}
In the main paper we note that the IRMv1 penalty of \citet{arjovsky2019invariant} can be shown to prefer large margins when applied with linear models to separable datasets. This can be shown when we apply the IRMv1 principle with exponentially decaying losses such as the logistic or the exponential loss. We characterize the condition on the losses below and then give the result using a technique similar to \cite{NIPS2003_0fe47339} who prove that exponentially decaying losses maximize margins under separable datasets.

For generality we do not assume anything about the data generating process (specifically, we do not assume the data is Gaussian as we do in the main paper) and also allow for more than two training environments. We do assume for simplicity that the datasets for each environment are of the same size, yet the proof can be easily adjusted to account for varying sizes. Let $S_e = \{(\x^e_i, y^e_i)\}_{i=1}^{m}$ be datasets for each environment $e\in{E_{\text{train}}}$, with $\cX=\reals^d, \Y=\{-1, 1\}$. Assume the pooled dataset $S=\{(\x_i^e, y_i^e)\}_{i,e}$ is linearly separable and we are learning with an $l_2$ regularized IRM, that is
\begin{align} \label{eq:irm_l2_reg}
     \mathcal{L}(\w; S) + \lambda_1\sum_{e}{\| \nabla_{v:v=1}\mathcal{L}(v\cdot\w; S^e) \|} + \lambda_2\|\w\|_2^2.
\end{align}
Here we defined the average loss over a dataset as $\mathcal{L}(\w; \tilde{S}) = \frac{1}{|\tilde{S}|}\sum_{(\x, y)\in{\tilde{S}}}{l(\w^\top\x y)}$. Our result holds for losses that satisfy the following conditions for any $\epsilon > 0$:
\begin{align}
    \lim_{t\rightarrow\infty}{\frac{l(t\cdot[1-\epsilon])}{l(t)}}&=\infty, \label{eq:logloss_condition}\\
    \lim_{t\rightarrow\infty}{\frac{\nabla_{s:s=t\cdot[1-\epsilon]} l(s)}{\nabla_{s=t}l(s)}}&=\infty. \label{eq:grad_logloss}
\end{align}

\paragraph{Relation to Previous Formal Results.} Previous works \citep{zhou2022sparse, lin2022bayesian} have noted that the IRMv1 penalty cannot distinguish between solutions that achieve $0$ loss. That is, they prove that the solution of the IRMv1 problem is not unique when $0$ loss is achievable and that the set of possible solutions coincides with the set of possible solutions for ERM. Yet here we are interested in the implicit bias of learning algorithms, hence we are interested in a more specific characterization of the IRMv1 solutions that is not provided by prior works. We ask whether out of the hyperplanes that separate $S$, will the IRMv1 principle find one that attains a margin that is considerably smaller than the attainable margin (hence our negative result would not imply its failure), or instead it finds a large-margin separator and hence our theory predicts its failure in learning a robust classifier?

While ideally we would like to give a characterization of the solutions towards SGD will converge, as in e.g. \citet{JMLR:v19:18-188}, the techniques used to gain such results are inapplicable to non-convex losses such as the IRMv1 penalty. Hence we turn to prove a different type of result, concerning the convergence of solutions for an $\ell_2$ regularized problem, as the regularization term vanishes. This has been used in previous work to gain intuition on the type of losses that lead to margin-maximizing solutions \citep{NIPS2003_0fe47339}.

\begin{claim}
Let $\hat{\w}(\lambda_2)$ be a minimizer of \eqref{eq:irm_l2_reg} with $\lambda_2 > 0$, where we assume that the empirical loss $l:\reals \rightarrow (0, \infty)$ is monotone non-increasing, and satisfies \eqref{eq:logloss_condition} and \eqref{eq:grad_logloss}. Any convergence point of $\frac{\hat{\w}(\lambda_2)}{\|\hat{\w}(\lambda_2)\|_2}$ as $\lambda_2 \rightarrow 0$ is a maximum margin classifier on $S$.
\end{claim}
\begin{proof}
We prove the claim in three steps.
\paragraph{Showing that $\lim_{\lambda_2\rightarrow 0}{\|\hat{\w}(\lambda_2)\|_2^2}=\infty$.}
To this end we note that the loss is strictly positive and due to \eqref{eq:logloss_condition} it approaches $0$ if and only if the margin approaches $\infty$. The margin can only grow unboundedly large if the weights also do. Therefore if a sequence $\hat{\w}(\lambda_2)$ approaches loss $0$ as
$\lambda_2\rightarrow 0$ then also $\|\hat{\w}(\lambda_2)\|\rightarrow \infty$. 
Now we will show that this must happen, thus concluding this part of the proof. Due to \eqref{eq:grad_logloss} we can also observe that the IRMv1 regularizer approaches $0$ as the margin approaches $\infty$. This holds since for any dataset $\tilde{S}$,
\begin{align*}
    \left\| \nabla_{v:v=1}\cL\left(v\cdot \w; \tilde{S}\right) \right\| = \left\| \frac{1}{|\tilde{S}|}\sum_{(\x_i, y_i)\in{\tilde{S}}}{\langle \w, \x_i \cdot y_i \rangle\cdot\nabla_{s:s=\langle \w, \x_i \cdot y_i \rangle}\ell\left( s \right)} \right\|.
\end{align*}
Invoking \eqref{eq:grad_logloss}, we may gather that for any $\epsilon > 0$ it holds that
\begin{align*}
  \lim_{t\rightarrow \infty}{\frac{t\cdot[1-\epsilon]\cdot \nabla_{s:s=t\cdot [1-\epsilon]}l(s)}{t\cdot \nabla_{s:s=t} l(s)}} = (1-\epsilon)\cdot \lim_{t\rightarrow \infty}{\frac{\nabla_{s:s=t\cdot [1-\epsilon]}l(s)}{\nabla_{s:s=t} l(s)}} = \infty.
\end{align*}
Due to non-increasing monotonicity of the loss, we see that $t\cdot \nabla_{s:s=t} l(s) < 0$ any $t>t_0$ for some $t_0$ and approaches $0$ as $t\rightarrow \infty$.\footnote{we should also note that the derivative of the loss does not approach $-\infty$ since it is bounded below by $0$, hence it must indeed be the case that the denominator approaches $0$.} This means that if the margin attained by some series of hyperplanes approaches $\infty$, then the value of the IRMv1 regularizer also approaches $0$. Now to see that $\| \hat{\w}(\lambda_2) \| \rightarrow \infty$ as $\lambda_2 \rightarrow 0$, we observe that the objective of this series as given by \eqref{eq:irm_l2_reg} must approach $0$. This holds since there is a series of hyperplanes whose objective approaches $0$ (e.g. any series whose norm grows sub-linearly with $\lambda^{-1}_2$ and separates the dataset with margin that grows to infinity). Hence any series whose objective does not approach $0$ cannot be a series of minimizers. As mentioned before, the loss can only approach $0$ if $\| \hat{\w}(\lambda_2) \|$ approaches $\infty$.

\paragraph{Maximum margin separators approach $0$ faster than others.} Let $\w_1$ and $\w_2$ be vectors on the unit sphere that define separating hyperplanes, and assume $\w_1$ achieves a larger margin than $\w_2$. That is, if we define for $k\in{\{1, 2\}}$
\begin{align*}
    s^{(k)}_{i,e} = \w_k^\top\x^e_iy^e_i,~ s^{(k)}_{e} = \min_i{s^{(k)}_{i,e}}, \text{ and } s^{(k)} = \min_{e}{s^{(k)}_e},
\end{align*}
then our assumption on the margins is that $s^{(1)} > s^{(2)}$. We will show that the direction of $\hat{\w}(\lambda_2)$ cannot approach $\w_2$ as $\lambda_2\rightarrow 0$. To this end, note that for some $t_0$ it holds that for all $t>t_0$ we have simultaneously that:
\begin{align}
    t\cdot s^{(1)} \nabla_{s:s=t\cdot s^{(1)}}l(s) &< (m^2 |E_{\text{train}}|)^{-1}t\cdot s^{(2)} \nabla_{s:s=t \cdot s^{(2)}}l(s) \label{eq:irm_bound_eq1}\\
    t\cdot s^{(1)}_{i,e} \nabla_{s:s=t \cdot s^{(1)}_{i,e}}l(s) &\geq t\cdot s^{(1)} \nabla_{s:s=t\cdot s^{(1)}}l(s) \quad \forall i, e \label{eq:irm_bound_eq2}
\end{align}
This is true since $s^{(1)}>s^{(2)}$, $s^{(1)}_{i,e}\geq s^{(1)}$ for all $i,e$, and \eqref{eq:grad_logloss} holds. We will now show that the value of the IRMv1 regularizer under $t\w_1$ is lower than that under $t\w_2$:
\begin{align*}
    \sum_{e}{\| \nabla_{v:v=1}\mathcal{L}(v\cdot\w_1; S^e) \|} &= \sum_{e}\left(\frac{1}{m}\sum_{i=1}^{m}t\cdot s^{(1)}_{i, e}\nabla_{s:s=t\cdot s^{(1)}_{i,e}}l(s)\right)^2 \\
    &\leq  |E_{\text{train}}|\cdot{(t\cdot s^{(1)} \nabla_{s:s=t\cdot s^{(1)}}l(s))^2} \\
    &< {(m^{-1}\cdot t\cdot s^{(2)} \nabla_{s:s=t\cdot  s^{(2)}}l(s))^2} \\
    &\leq \sum_{e}\left(\frac{1}{m}\sum_{i=1}^{m}t\cdot s^{(2)}_{i, e}\nabla_{s:s=t\cdot s^{(2)}_{i,e}}l(s)\right)^2 \\
    &=\sum_{e}{\| \nabla_{v:v=1}\mathcal{L}(v\cdot\w_2; S^e) \|}.
\end{align*}
The third inequality is true since for a separating hyperplane, all the summands are negative, thus if we only add summands then the square of the entire sum becomes larger. The second one is due to \eqref{eq:irm_bound_eq1}, while the first inequality is due to \eqref{eq:irm_bound_eq2} when we again use the negativity of the summands. 

Following the proof of \cite{NIPS2003_0fe47339}, we can gather that the above also holds when we replace the 
multiplication between gradients of the loss and the margin, simply to the value the loss. Taken together, this means that $t\w_1$ achieves a lower loss than $t\w_2$ for all $t> t_0$ for some $t_0$.

\paragraph{Showing that only max-margin classifiers can be limit points.} To finish the proof we simply note that if $\w_1$ is a convergence point of $\hat{\w}(\lambda_2) / \| \hat{\w}(\lambda_2) \|$ and it achieves margin $\gamma_1$, while $\w_2$ is another vector on the unit sphere that attains margin $\gamma_2 > \gamma_1$. Then there is also a neighborhood around $\w_1$, $N_{\w_1} = \{\w : \|\w\|=1, \|\w-\w_1\|\leq \delta\}$ with sufficiently small $\delta$, such that each $\w\in{N_{\w_1}}$ attains margin at most $\gamma_2 - \epsilon > \gamma_1$. Therefore by our previous paragraph we know that for some $t_0$ and all $t>t_0$ then $t\w_2$ attains a lower loss than $t\w$ for all $\w\in{N_{\w_1}}$, which means $\w_1$ cannot be a convergence point of $\hat{\w}(\lambda_2) / \| \hat{\w}(\lambda_2) \|$ (since the items of this series are kept out of a neighborhood $N_{\w_1}$ around $\w_1$). This is a contradiction to $\w_1$ being a convergence point, meaning $\w_2$ cannot attain a larger margin than $\w_1$.
\end{proof}

Our simulations indicate that indeed, when we apply unregularized IRMv1 to the synthetic setting we study, it finds a separator that is very close to the max-margin solution. This also holds for some other penalties such as VREx and we believe that the proof above can be adapted to this loss and several other ones, which we leave for future work. We illustrate this in \yw{!!TODO: add figure!!} ref by plotting the cosine similarity of the learned hyperplane with the hyperplane found by an SVM classifier (trained with a hinge loss using LinearSVC from sklearn \citep{scikit-learn} with $C=1000$) as $d$ grows and we approach the interpolating regime. We note that while some other methods such as CORAL and MMD do not maximize margin, they still separate the data with non-vanishing margin and indeed fail achieve worst robust accuracy as the dimension grows.

\section{Experimental Details for Waterbirds Dataset} \label{sec:waterbirds_details}
Here we elaborate on experimental details in our Waterbirds experiment that were left out from the main paper due to lack of space. The dataset is split into training, validation and test sets with 4795, 1199 and 5794 images in each set, respectively. We follow previous work \citep{pmlr-v119-sagawa20a, veldanda2022fairness} in defining a binary task in which waterbirds is the positive class and landbirds are the negative class, and using the following random features setup: for every image, a fixed pre-trained ResNet-18 model is used to extract a $d_{\mathrm{rep}}$-dimensional feature vector $\x'$ ($d_{\mathrm{rep}}=512$). This feature vector is then converted into an $d$-dimensional 
feature vector $\x = \text{ReLU}(U\x')$, where $U \in \R^{d\times d_{\mathrm{rep}}}$ is a random matrix with Gaussian entries. Finally, a logistic regression classifier is trained on $\x$. The extent of over-parameterization in this setup is controlled by varying $d$, the dimensionality of $\x$. In our experiments we vary $d$ from $50$ to $2500$, with interpolation empirically observed at $d=1000$ (which we refer to as the interpolation threshold).

For all the experiments we use the Adam optimizer, a batch size of 128 and a learning rate schedule with initial rate of 0.01 and a decay factor of 10 for every 10,000 gradient steps. Every experiment is repeated 25 times and results are reported over all runs. For the baseline model we train for a total of 30,000 gradient steps whereas for our two-phased algorithm we use 15,000 gradient steps for each model in Phase A and an additional 250 steps for Phase B. 



\end{document}